\documentclass{article}
\usepackage[numbers]{natbib}
\usepackage[final]{neurips_2022}
\usepackage[utf8]{inputenc}
\usepackage[T1]{fontenc}

\usepackage{microtype}
\usepackage[pagebackref=true]{hyperref}
\hypersetup{colorlinks}
\hypersetup{linkcolor=[rgb]{.7,0,0}}
\hypersetup{citecolor=[rgb]{0,.7,0}}
\hypersetup{urlcolor=blue}
\usepackage{amssymb,amsfonts,amsmath,amsthm,amscd,dsfont,mathrsfs,pifont}
\usepackage{blkarray}
\usepackage{graphicx,float,psfrag,epsfig,color}
\usepackage{qtree}
\usepackage{comment}
\usepackage{xcolor}
\usepackage{caption}
\usepackage{subcaption}
\usepackage{booktabs}

\newcommand\independent{\protect\mathpalette{\protect\independent}{\perp}}
\def\independent#1#2{\mathrel{\rlap{$#1#2$}\mkern2mu{#1#2}}}

\newcommand{\pr}{\mathbb{P}}
\newcommand{\E}{\mathbb{E}}
\DeclareMathOperator{\CP}{CP}

\newcommand{\mR}{\mathbb{R}}

\newcommand{\mF}{\mathbb{F}}

\DeclareMathOperator{\ReLU}{ReLU}

\DeclareMathOperator{\INAL}{INAL}
\DeclareMathOperator{\sign}{sgn}

\DeclareMathOperator{\orb}{orb}
\DeclareMathOperator{\NS}{NS}

\DeclareMathOperator{\Stab}{Stab}
\DeclareMathOperator{\NN}{NN}
\newcommand{\cF}{\mathcal{F}}
\newcommand{\cX}{\mathcal{X}}
\newcommand{\cU}{\mathcal{U}}
\newcommand{\cN}{\mathcal{N}}
\newcommand{\bR}{\mathbb{R}}
\newcommand{\bN}{\mathbb{N}}
\newcommand{\bF}{\mathbb{F}}
\newcommand{\cY}{\mathcal{Y}}

\theoremstyle{plain}
\newtheorem{defin}{Definition}
\theoremstyle{plain}
\newtheorem{thm}{Theorem}
\theoremstyle{plain}
\theoremstyle{plain}
\theoremstyle{plain}
\theoremstyle{plain}
\newtheorem{prop}{Proposition}
\theoremstyle{plain}
\theoremstyle{plain}
\newtheorem{lemma}{Lemma}
\theoremstyle{plain}
\theoremstyle{plain}
\theoremstyle{remark}
\newtheorem{remark}{Remark}
\theoremstyle{remark}
\theoremstyle{plain}
\theoremstyle{plain}
\newtheorem{assumption}{Assumption}

\definecolor{DSgray}{cmyk}{0,0,0,0.7}
\definecolor{DSred}{cmyk}{0,0.7,0,0.7}
\definecolor{DSblue}{cmyk}{0.7,0,0,0}

\title{
Learning to Reason with Neural Networks: Generalization, Unseen Data and Boolean Measures}
\author{%
  Emmanuel Abbe\thanks{Authors are in alphabetical order.} \\
  EPFL
  \And
  Samy Bengio \\
  Apple
  \And
  Elisabetta Cornacchia \\
  EPFL 
  \And
  Jon Kleinberg \\
  Cornell University 
  \And
  Aryo Lotfi \\
  EPFL 
  \And
  Maithra Raghu \\
  Google Research 
  \And
  Chiyuan Zhang \\
  Google Research 
}

\begin{document}
\maketitle

\begin{abstract}
    This paper considers the Pointer Value Retrieval (PVR) benchmark introduced in~\cite{Zhang2021PointerVR}, where a `reasoning' function acts on a string of digits to produce the label. More generally, the paper considers the learning of logical functions with gradient descent (GD) on neural networks. It is first shown that in order to learn logical functions with gradient descent on symmetric neural networks, the generalization error can be lower-bounded in terms of the {\it noise-stability} of the target function, supporting a conjecture made in~\cite{Zhang2021PointerVR}. It is then shown that in the distribution shift setting, when the data withholding corresponds to freezing a single feature (referred to as canonical holdout), the generalization error of gradient descent admits a tight characterization in terms of the {\it Boolean influence} for several relevant architectures. This is shown on linear models and supported experimentally on other models such as MLPs and Transformers. In particular, this puts forward the hypothesis that for such architectures and for learning logical functions such as PVR functions, GD tends to have an {\it implicit bias towards low-degree representations}, which in turn gives the Boolean influence for the generalization error under quadratic loss.
\end{abstract}

\section{Introduction}
Recently~\cite{Zhang2021PointerVR} introduced the pointer value retrieval (PVR) benchmark. This benchmark consists of a supervised learning task on MNIST~\cite{lecun2010mnist} digits with a `logical' or `reasoning' component in the label generation. More specifically, the functions to be learned are defined on MNIST digits organized either sequentially or on a grid, and the label is generated by applying some `reasoning' on these digits, with a specific digit acting as a pointer on a subset of other digits from which a logical/Boolean function is computed to generate the label.

For instance, consider the PVR setting for binary digits in the string format, where a string of MNIST digits is used as input. Consider in particular the case where only 0 and 1 digits are used, such as in the example of Figure~\ref{fig:pvr-example}. The label of this string is defined as follows: the first 3 bits $101$ define the pointer in binary expansion, and the pointer points to a window of a given length, say 2 in this example. Specifically, the pointer points at the first bit of the window. To generate the label, one has to thus look the 6th window\footnote{pointer $000$ points at the first window, pointer $001$ at the second window, and so on. Thus pointer $101$, that is equal to $5$ in binary expansion, points at the 6th window.} of length 2 given by $11$, and there the label is produced by applying some fixed function, such as the parity (so the label would be 0 in this example). In \cite{Zhang2021PointerVR}, the PVR benchmark is also defined for matrices of digits rather than strings; we focus here on the string version that captures all the purposes of the PVR study.
\begin{figure}[t]
\centering
\includegraphics[width=0.8\textwidth]{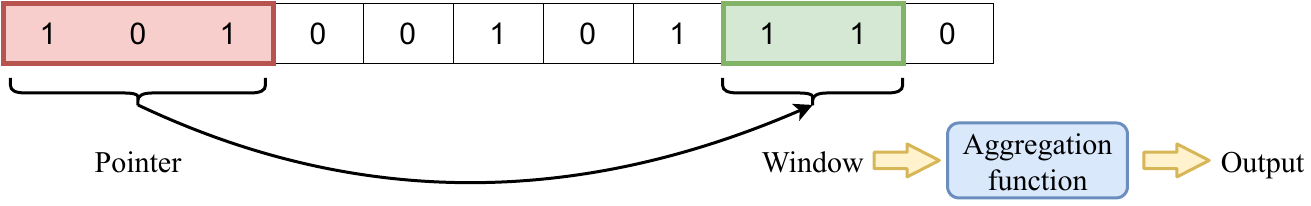}
\caption{An example of a PVR function with a window size of 2. The first $3$ bits are the pointer, which points to a window in the subsequent bits. Specifically, the number indicated by the pointer bits in binary expansion gives the position of the first bit of the window. The label is then produced by applying some fixed aggregation function to the window bits (e.g., parity, majority-vote, etc.).}
\label{fig:pvr-example}
\vspace{-2em}
\end{figure}
This benchmark is introduced to understand the limits of deep learning on tasks that go beyond classical image recognition, investigating in particular the trade-off between memorization and reasoning by acting with a particular distribution shift at testing (see further details below).  

In order to learn such PVR functions, one has to first learn the digit identification and then the logical component on these digits. Handling both tasks successfully at once is naturally more demanding than succeeding at the latter assuming the first is successful. Here, we focus on the `logical component' as a necessary component to learn, and this corresponds to learning a Boolean function. The overall function that maps the pixels of an image to its label in the PVR is of course also a Boolean function (like any computer-encoded function), but the structural properties of such meta-functions are more challenging to describe, and left for future work. In any case, to understand the limits of deep learning on such benchmarks, we focus on investigating first the limits of deep learning on the logical/Boolean component. 

We next re-state formally one of the PVR benchmarks from~\cite{Zhang2021PointerVR}, focusing on binary digits for simplicity. We will use the alphabet $\{0,1\}$ to describe the benchmark to connect to the MNIST dataset, but will later switch to the alphabet $\{+1,-1\}$ for the problem of learning Boolean functions with neural networks. Recall for $d \in \bN$, $\bF_2^d = \{0,1\}^d$.




\begin{defin}[Boolean PVR with \textit{sliding} windows] \label{def:PVR}
The input $x$ consists of $n$ bits and the label is generated as
\begin{align}
f(x_1,\dots,x_{n})= g(x_{P(x^p)}, \ldots, x_{P(x^p)+w-1}).
\end{align}
where $p$ is the number of bits in the pointer, $w$ is the window size, $g: \mF_2^w \to \mR$ is the aggregation function, $P: \mF_2^p \to [p+1:n]$ is the pointer map, $x_j$ denotes the bit in position $j$ and $x^p =(x_1,...,x_p)$ denotes the pointer bits. We often set $n = p + 2^p$, and hence the last window starts at the last digit.

In words, the first $p$ bits give a pointer to the beginning of a window of $w$ consecutive bits, and the label is produced by applying a Boolean function $g$ on these $w$ bits. 

\end{defin}

\begin{remark}
If $w > 1$, for some values of the pointer, $P(x^p) + w - 1$ would exceed the dimension of input $n$. 
This issue can be solved by using cyclic indices or by using non-overlapping  windows as defined in Appendix \ref{sec:NSforPVR}. However, in the experiments of this paper, we mainly truncate windows (if necessary) in order to capture the underlying asymmetries (e.g., the last windows have smaller sizes).
\end{remark}
In this paper we consider the problem of learning in the holdout setting, i.e., when some data are withheld (or `unseen') during training. We focus on a specific type of holdout setting, that we call `canonical holdout', which we define here. Extensions to other types of holdout are also of interest, and left for future work.


\begin{defin}[Canonical holdout] \label{def:canonical_distr_shift}
Let $\mathcal{F}$ be the class of Boolean functions on $n$ bits and let $f$ be a specific function in that class. For $k \in [n]$, consider the problem of learning $\mathcal{F}$ from samples $(X_{-k},f(X_{-k}))$,
where the $X_{-k}$'s are independently drawn from the distribution that freezes component $k$ to $1$ and that draws the other components i.i.d.\ Bernoulli$(1/2)$. Let $\tilde{f}_{-k}$ be the function learned under this training distribution. We are interested in the generalization error with square loss when the $k$-th bit is not frozen at testing, i.e.,
\begin{align}
    \mathrm{gen}(\mathcal{F},\tilde{f}_{-k}) = \frac{1}{2}\mathbb{E}_{X \sim_U {\mathbb{F}_2^n}, f \sim_U \mathcal{F}}(f(X) - \tilde{f}_{-k}(X))^2,
\end{align}
where by $X \sim_U {\mathbb{F}_2^n} $ we mean that $X$ is chosen uniformly at random from ${\mathbb{F}_2^n} $, and similarly for $f \sim_U \mathcal{F} $.
We denote by $\tilde f$ the function learned in the case where the $k$-th component is not frozen at training, i.e., when the train and test distributions are both uniform on the Boolean hypercube, and use the following notation for the corresponding generalization error:
\begin{align}
    \mathrm{gen}(\mathcal{F},\tilde{f}) = \frac{1}{2}\mathbb{E}_{X \sim_U \mathbb{F}_2^n, f \sim_U \mathcal{F}} (f(X) - \tilde{f}(X))^2.
\end{align}
\end{defin}
The `canonical' holdout is thus a special case of holdout where a single feature/bit is frozen at training. In \cite{Zhang2021PointerVR}, different types of holdout are allowed where a given string is absent from the windows rather than a given bit as in the canonical version. We believe that the canonical holdout still captures the essence of the PVR benchmark: other windows will get to see the strings that are withheld due to the bit freezing, and thus the goal is still to investigate whether neural networks trained by GD will manage to patch together the different pieces. 

Finally, we will often consider neural network architectures that have invariances on the input, such as permutation invariance. In the PVR setting, this means that we do not assume that the learner has knowledge of which bits are in the window. In such cases, instead of learning a class of function $\mathcal{F}$, one can equivalently talk about learning a single function $f$ (with the implicit $\mathcal{F}$ defined as the orbit of $f$ through the permutation group), and define
\begin{align*}
    \mathrm{gen}(f,\tilde{f}_{-k}) = \frac{1}{2}\mathbb{E}_{X \sim_U \mathbb{F}_2^n}(f(X) - \tilde{f}_{-k}(X))^2 \quad \text{ and } \quad \mathrm{gen}(f,\tilde{f}) = \frac{1}{2}\mathbb{E}_{X \sim_U \mathbb{F}_2^n} (f(X) - \tilde{f}(X))^2.
\end{align*}

\subsection{Contributions of this paper}

The contributions of this paper are:
\begin{enumerate}
    \item In the matched setting (i.e., train and test distributions are matching), we prove a lower-bound (Theorem \ref{thm:distr_match}) on the generalization error for gradient descent\footnote{This part relies on population gradients with polynomial gradient precision as in \cite{AS20,abbe2021power}} (GD), that degrades for functions having large {\it noise-sensitivity} (or low noise-stability), supporting thereby with a formal result a conjecture put forward in \cite{Zhang2021PointerVR}. 
    \item In the mismatched setting, specifically in the canonical holdout where a single feature is frozen at training and released to be uniformly distributed at testing, we hypothesize that {\it (S)GD on the square loss\footnote{We do not expect the square loss to be critical, but it makes the connection to the influence more explicit.} and on certain network architectures such as MLPs and Transformers has an implicit bias towards low-degree representations when learning logical functions such as Boolean PVR functions.} This gives a new insight into the implicit bias study of GD on neural networks for the case of logical functions.  

We then show (Lemma \ref{lem:Boolean_inf}) that under this hypothesis, the generalization error in the canonical holdout setting is given by the Boolean influence, a central notion in Boolean Fourier analysis.
    \item We provide experiments supporting this hypothesis for various target functions and architectures such as MLPs and the Transformers \cite{vaswani2017attention-transformer} in Section \ref{sec:exps} and Appendix \ref{appendix:experiments}. These rely on mini-batch stochastic gradient descent (see Section \ref{sec:exps} for more details). 
    \item We establish formally the hypothesis for GD and linear regression models in Section \ref{sec:linearmodels}, and conduct experiments to support it on multi-layer neural networks of large depths and small initialization scales in Section \ref{sec:linearmodels} (with further arguments in Appendix \ref{app:linear-nns}).
\end{enumerate}

\subsection{The low-degree bias hypothesis: illustration} \label{sec:lowdegreebias}
We now discuss the hypothesis put forward in this paper that ``GD on certain architectures such as MLPs and Transformers has an implicit bias towards lower-degree representations when learning PVR and, more generally, logical functions.'' Before starting our illustration let us recall some basic notions from Boolean analysis (we refer to~\cite{o'donnell_2014} for details). Any Boolean function $f:\{\pm 1\}^n \to \bR$ can be written in terms of its Fourier-Walsh transform: $f(x) = \sum_{T \subseteq [n]} \hat f(T) \chi_T(x)$, where $\chi_T(x)=\prod_{i \in T}x_i $ and $   \hat f (T) = 2^{-n}\sum_{x \in \{\pm 1\}^n} f(x) \chi_T(x)  $ are respectively the basis elements and the coefficients of the Fourier-Walsh transform of $f$. 
\begin{defin}[Boolean influence~\cite{o'donnell_2014}]
Let $f: \{\pm 1\}^n \to \mathbb{R}$ be a Boolean function and let $\hat f$ be its Fourier-Walsh transform. The Boolean influence of variable $k \in [n]$ on $f$ is defined by  
$ \mathrm{Inf}_{k}(f):=\sum_{T \subseteq [n] : k \in T} \hat{f}(T)^2.$
In particular, if $f:\{\pm 1 \}^n \to \{ \pm 1 \}$,  $\mathrm{Inf}_{k}(f)=\mathbb{P}(f(X)\ne f(X+e_k)),$ where $X+e_k$ corresponds to the vector obtained by flipping the $k$-th component of $X$.  
\end{defin}

Consider the following example of a PVR function with a 1-bit pointer, 2 overlapping windows of length 2, and parity for the aggregation function. We consider $f : \{\pm 1\}^n \to \{ \pm 1\}$ (i.e., with $\pm 1$ variables instead of $0,1$, to simplify the expressions), with $f$ given by 
\begin{align}
f(x_1,x_2,x_3,x_4)&=
\frac{1+x_1}{2}x_2x_3 +\frac{1-x_1}{2}x_3x_4 .
\end{align}
We can rewrite $f$ in terms of its Fourier-Walsh expansion (i.e., pulling out all the multivariate monomials appearing in the function, see Section~\ref{sec:influence_and_generalization} for more details), which gives
\begin{align}
f(x_1,x_2,x_3,x_4)&=\frac{1}{2}x_2x_3+\frac{1}{2}x_3x_4+\frac{1}{2}x_1x_2x_3- \frac{1}{2}x_1x_3x_4 .
\end{align}


Consider now training a neural network such as a Transformer as in Section~\ref{sec:exps} on this function, with quadratic loss, and a canonical holdout corresponding to freezing $x_2=1$ at training.
Under this holdout, and under the `low-degree implicit bias' hypothesis, the low-degree monomials are learned first (see experiments in Section~\ref{sec:exps}), resulting in the following function being learned at training: 
\begin{align}
f_{-2}(x_1,x_2,x_3,x_4) &= \frac{1}{2}x_3+\frac{1}{2}x_3x_4+\frac{1}{2}x_1x_3- \frac{1}{2}x_1x_3x_4  =\frac{1+x_1}{2}x_3 +\frac{1-x_1}{2}x_3x_4. 
\end{align}
Thus, according to Lemma \ref{lem:Boolean_inf} proved in Appendix \ref{sec:proof_mismatch}, the generalization error is given by
\begin{align}
\frac{1}{2}\E (f(X)- f_{-2}(X) )^2 = \pr(f(X)\ne f(X+e_2))= \frac{1}{2}, 
\end{align}
which is the probability that flipping the frozen coordinate changes the value of the target function, i.e., the Boolean influence (where we denoted by $X+e_2$ the vector obtained by flipping the second entry in $X$).
As shown in this paper, neural networks tend to follow this trend quite closely, and we can prove this hypothesis on simple linear models. Notice that an ERM-minimizing function could have taken a more general form than a degree minimizing function, i.e., 
\begin{align}
f_{-2}^{\rm ERM}(x) := \frac{1+x_2}{2} f_{-2}(x) +\frac{1-x_2}{2} r(x) 
\end{align}
for any choice of $r: \{\pm 1\}^4 \to \{ \pm 1 \}$. In the special case of $r=0$, $f_{-2}^{\rm ERM}$ corresponds to $f_{-2}$ (the low-degree representation). 
For instance, among such ERM-minimizers, one can check that the minimum $\ell_2$-norm interpolating solution would be given by
\begin{equation}
     f_{-2}^{\ell_2}(x) := \frac{1}{4}(x_3+x_2x_3)+\frac{1}{4}(x_3x_4+x_2x_3x_4)+\frac{1}{4}(x_1x_3 + x_1x_2x_3)- \frac{1}{4}(x_1x_3x_4 + x_1x_2x_3x_4).
\end{equation}
This gives a generalization error of $\frac{1}{2}\E (f(X)- f_{-2}^{\ell_2}(X))^2 = 4(\frac{1}{4})^2 = \frac{1}{4}$, i.e., half the error of $f_{-2}$, yet still bounded away from 0. 

In order to improve on this, under the same canonical holdout with $x_2=1$, one would like to rely on a type of minimum description length bias, since describing $f$ may be more efficient than $f_{-2}$ due to the stronger symmetries of $f$. Namely, $f$ corresponds to taking the parity on the middle two bits if $x_1=1$, and on the last two bits otherwise. On the other hand, $f_{-2}$ requires changing the function depending on $x_1=1$ or $x_1=-1$, since it is once the function $x_3$ and once the function $x_3x_4$. 
So an implicit bias that would exploit such symmetries, featuring in PVR tasks, would give a different solution than the low-degree implicit bias, and could result in lower generalization error. We leave to future work to investigate this `symmetry compensation' procedure. 
\subsection{Related literature}
\paragraph{GD lower bounds.}
Several works have investigated the difficulties of learning a class of functions with descent algorithms and identified measures for the complexity in this setting. In particular,~\cite{blum1994weakly, kearns1998efficient} prove that the larger the statistical dimension of a function class is, the more difficult it is for an statistical query (SQ) algorithm to learn. We refer to~\cite{abbe2021power,Das2019learnability} for further references on SQ-like lower bounds.
For GD with mini-batch and polynomial gradient precision,~\cite{AS20} uses the $m$-Cross-Predictability (CP) of a class of functions to show that classes  sufficiently small CP and large batch-size are not efficiently learnable. This is further generalized in \cite{abbe2021power} to a broader range of batch size and gradient precision. 
In~\cite{arora2018stronger}, the noise-stability for deep neural networks is defined as the stability of each layer’s computation to noise injected at lower layers, and is used to show correctness of a compression framework. 
However, no bounds on the generalization error of GD depending on the noise stability of the target function is derived in this work. 
The noise-stability, the statistical dimension and the cross-predictability are measures for some given function or function class. One can obtain measures that are defined for a dataset and a network architecture. For that purpose, \cite{AbbeINAL} introduced the ``Initial Alignment'' (INAL) between the target function and the neural network at initialization and proves that for fully connected networks with Gaussian i.i.d. initialization, if the INAL is negligible, then GD will not learn efficiently. We remark that the problem of approximating and learning Boolean functions appear in other areas as well. For instance,~\cite{heidari2019boolean} considered the problem of approximating Boolean functions, using functions coming from restricted classes (namely k-juntas and linear Boolean functions);~\cite{jha2019explaining} proposes two algorithms to learn sparse Boolean formulae; and ~\cite{udovenko2021milp} proposes techniques for modeling Boolean functions by mixed-integer linear inequalities.


\paragraph{Implicit bias.}
The implicit bias of neural networks trained with gradient descent has been extensively studied in recent years~\cite{neyshabur2014search,neyshabur2017exploring,moroshko2020implicit,ji2019characterizing,gunasekar2018characterizing}. In particular, \cite{soudry2017implicit} proved that gradient descent on linearly-separable binary classification problems converges to the maximum $\ell_2$-margin direction. Several subsequent works studied the implicit bias in classification tasks on various networks architectures, e.g., homogeneous networks~\cite{lyu2019gradient}, two layers networks in the mean field regime~\cite{chizat2020implicit}, linear and $\ReLU$ fully connected networks~\cite{vardi2021margin}, and convolutional linear networks~\cite{gunasekar2018implicit}. Among regression tasks, the problem of implicit bias has been analysed for matrix factorization tasks~\cite{gunasekar2017implicit,razin2020implicit,arora2019implicit}, and also gradient flow on diagonal networks~\cite{pesme2021implicit}.
However, all these works consider functions with real inputs, instead of logical functions which are the focus of this work. On the other hand, as discussed in Section~\ref{sec:lowdegreebias}, the Boolean influence generalization characterization reflects the implicit bias of GD on neural networks to learn low-degree representations. Similar types of phenomena can implicitly be found in~\cite{Xu2018frequency,xu2019frequency,rahaman2019spectral}, in particular as the ``spectral bias'' in the context of real valued functions decomposed in the classical Fourier basis (where the notion of lower degree is replaced by low frequencies). In ~\cite{abbe2021staircase,mergedstaircase}, the case of Boolean functions is considered as in this paper, and it is established that for various `regular' architectures (having some symmetry in their layers), gradient descent can learn target functions that satisfy a certain `staircase' property. 
However, these papers do not investigate lower-bounds in terms of noise stability, neither distribution shift as in the canonical holdout. We note that, in the context of Boolean functions, the problem of regularizing a learning algorithm to avoid overfitting appears in other areas of research, e.g., in the analysis of fitness functions in biology~\cite{epistaticnet}.

\paragraph{Distribution shift.} 
Many works were aimed at characterizing when a classifier trained on a training distribution (also called the ``source'' distribution) performs well on a different test domain (also called the ``target'' distribution)~\cite{datasetshift2009}. On the theoretical side,~\cite{BenDavid2018theory} obtains a bound of the target error of a general empirical risk minimization algorithm in terms of the source error and the divergence between the source and target distribution, in a setting where the algorithm has access to a large dataset from the source distribution and few samples from the target distribution. 
We refer to~\cite{shen2017wasserstein} for a further result in a similar setting. 
Instead, in this work we focus on gradient descent on neural networks in the setting where no data from the target distribution is accessible. 
On the empirical level, several benchmarks have been proposed to evaluate performance for a wide range of models and distribution shifts~\cite{wiles2022a,miller2021accuracyontheline,sagawa2021extendingwilds}.
Despite this significant body of works on distribution shift, we did not find works that related the generalization error under holdout shift in terms of the Boolean influence.

\section{Matched setting: a formal lower-bound on noise stability and generalization}
Our first result provides a lower bound on the generalization error in the setting where the training and test distributions of the inputs are both uniform on the Boolean hypercube, i.e., the ``matched setting''. We directly relate the generalization error achieved by GD, or SGD with large batch-size, to the complexity of the task, providing theoretical support to the empirical claim in~\cite{Zhang2021PointerVR}, that complex tasks are more difficult to learn for GD/SGD on neural networks. The latter work used the noise sensitivity as a dual measure of the target function complexity, whereas we use here the noise stability ($\Stab_{\delta}[f]$).
\begin{defin}[Noise stability] \label{def:stability}
Let $f:\{\pm 1\}^n \to \bR$ and $\delta \in [0,1/2]$. Let $X$ be uniformly distributed on the Boolean $n$-dimensional hypercube, and let $Y$ be formed from $X$ by flipping each bit independently with probability $\delta$. We define the $\delta$-noise stability of $f$ by
    $\Stab_{\delta}[f] := \E_{(X,Y)}[f(X) \cdot f(Y)]$.
\end{defin}
\noindent
Intuitively, $\Stab_\delta[f]$ measures how stable the output of $f$ is to a perturbation that flips each input bit with probability $\delta$, independently. The noise stability can be easily related to the noise sensitivity $\NS_\delta[f]$\footnote{Specifically, for binary-valued functions, $\NS_\delta[f] = \frac 12 - \frac 12 \Stab_\delta[f]$.}, used in~\cite{Zhang2021PointerVR}.

The generalization error depends as well on the complexity of the network, which is quantified in terms of the number of edges in the network, the number of time steps, and gradients precision in the gradient descent algorithm. We give one more definition before stating our result. 
\begin{defin}[N-Extension] \label{def:Nextension} For a function $f : \bR^n \to \bR$ and for $N>n$, we define its $N$-extension $\bar f: \bR^N \to \bR$ as
   $ \bar f(x_1,x_2,...,x_n,x_{n+1},...,x_N) = f(x_1,x_2,...,x_n).$
\end{defin}
\begin{thm} \label{thm:distr_match}
Consider a fully connected neural network of size $E$ with initialization that is target agnostic in the sense that it is invariant under permutations of the input\footnote{One could consider other groups of invariances than the full permutation group; this is left for future work.}. Let $f: \{\pm 1\}^n \to  \{\pm 1\} $ be a balanced target function. Let $\bar f$ be the $2n$-extension of $f$ as defined in Definition~\ref{def:Nextension}. Let $f_{\mathrm{NN}}^{(t)}$ be the output of noisy-GD with gradient range $A$, batch-size $b$, learning rate $\gamma$ and noise scale $\sigma$ (see Def.~\ref{def:noisyGD}) after $t$ time steps. Then, for $\delta$ small enough\footnote{Generally, this holds for any $\delta$ such that $CP({\rm orb}(\bar f)) \le \Stab_{\delta}[f]$; in particular, this holds for $\delta<1/4$ under input doubling or for some $\delta>0$ under non-extremal and non-dense assumptions.} and for $b$ large enough\footnote{This holds for $b \geq 1/CP({\rm orb}(\bar f))$.}, the generalization error satisfies
\begin{align}
{\rm gen}\left(\bar f, f_{\mathrm{NN}}^{(t)}\right)  \ge 1/2 -\frac{\gamma t \sqrt{E} A}{\sigma} \cdot \Stab_{\delta}[f]^{1/4}.
\end{align}
\end{thm}
\noindent
Theorem~\ref{thm:distr_match} states a lower bound for learning in the extended input space. 
We use the input doubling to guarantee  that the hypothesis class $\mathcal{F}$ resulting from the orbit of $\bar f$, i.e., $\orb(\bar f) = \{ \bar f \circ  \pi : \pi \in S_{2n} \}$, is not degenerate to a single function, as this could then be learned using a proper choice of the initialization (i.e., one can simply set the weights of the neural network at initialization to represent the unique function, if the network has enough expressivity). Instead, the input doubling prohibits such representation shortcut and ensures that the structural properties of the function is what creates the difficulty of learning, irrespective of the choice of the initialization. For instance, consider the full parity function on all of the input bits. The orbit of this function contains only that specific function, and one can learn that function by choosing a proper initialization on a neural net of depth 2. 
However, with an input doubling, this function becomes hard to learn {\it no matter what} the initialization is~\cite{kearns1998efficient,AS20}. 
One can remove this input doubling requirement by assuming that $f$ is non-extremal (i.e., no terms of degree $\theta(n)$ in the Fourier basis) and non-dense (i.e., poly$(n)$-sized Fourier spectrum), see Appendix~\ref{sec:no_input_augmentation}.

The proof of Theorem~\ref{thm:distr_match} uses~\cite{AS20} which obtains a lower-bound in terms of the cross-predictability (CP) (instead of the noise stability). The CP is a complexity measure of a \emph{class} of functions, rather than a single function. For the orbit of $\bar f$, the CP is defined as
$\CP(\orb(\bar f)) = \E_{\pi \sim_U S_{2n} } [\E_{X \sim_U {\mathbb{F}_2^{2n}}} [ \bar f(X) \cdot \bar f \circ \pi (X) ]^2 ]$.
We give a more general definition of CP in Appendix~\ref{sec:proof_distr_match} (we believe that the CP also extends to other invariances than permutations and non i.i.d. distributions). 

Theorem~\ref{thm:distr_match} states that if the target function is highly noise unstable, specifically if there exists $\delta< 1/4$ such that $ \Stab_\delta[f]$ decreases faster than any inverse polynomial in $n$, then GD will not learn the $2n$-extension of $f$ (or $f$ itself if it is non-extremal/dense) in polynomial time and with a polynomially sized neural network initialized at random. So in that sense, the noise-stability gives a proxy to generalization error, as observed in~\cite{Zhang2021PointerVR}. More specifically, this result is about failure of the weakest form of learning no matter what the architecture is. One could also consider `regular' architectures (such as with isotropic layers) and stronger learning requirements; for this it is known that the `Fourier leap' from~\cite{abbe2021staircase} is a relevant complexity measure, and we leave  investigations of regular architectures to future work. More details of the proof of Theorem~\ref{thm:distr_match} are in Appendix~\ref{sec:proof_distr_match}. In Appendix~\ref{sec:NSforPVR}, we explain how the noise stability of PVR functions can be computed, and the implications to the result of Theorem~\ref{thm:distr_match}.

\section{Mismatched setting} \label{sec:mismatch_setting}
In this section, we focus on the generalization error under the distribution shift setting, more specifically the canonical holdout setting defined in Definition~\ref{def:canonical_distr_shift}. Namely, assume that at training component $k$ is frozen to $1$, and assume it to be released to $\text{Unif}\{\pm 1 \}$ at testing.  
Our experiments show that in this setting, for some relevant architectures, the generalization error is close to the value of a standard measure in Boolean analysis, namely the Boolean influence of variable $k$ on $f$.

\subsection{Boolean influence and generalization} \label{sec:influence_and_generalization}
To explain the connection between generalization error and Boolean influence, we start with a simple lemma relating the Boolean influence to the $\ell_2$-distance between the true target function and the function obtained by freezing component $k$ (which we call the ``frozen function'').
\begin{lemma} \label{lem:Boolean_inf}
Let $f: \{ \pm 1\}^n \to \bR$ be a Boolean function and let $f_{-k}$ be defined as $f_{-k}(x):=f(x_{-k})$ where $x_{-k}(i)=1$ if $i = k$ and $x_{-k}(i)=x(i)$ otherwise. 
Then,
  $\frac{1}{2} \mathbb{E}_{X}  (f(X)-f_{-k}(X) )^2 = \mathrm{Inf}_{k}(f).$
\end{lemma}
The proof of Lemma \ref{lem:Boolean_inf} can be found in Appendix~\ref{sec:proof_mismatch}. 
In Section~\ref{sec:exps}, we present experiments that demonstrate the relation between the Boolean influence and the generalization error for different architectures. In Section~\ref{sec:linearmodels}, we focus on linear models, namely, linear regression and linear networks.

\subsection{Experiments} \label{sec:exps}

We consider three architectures for our experiments: multi-layer perceptron (MLP) with 4 hidden layers, the Transformer~\cite{vaswani2017attention-transformer}, and MLP-Mixer~\cite{tolstikhin2021mlpmixer}. 
For each architecture, we freeze different coordinates separately and evaluate our models. In other words, 
we train the model while freezing coordinate 1, then coordinate 2 and so on, until coordinate $n$ and compare the generalization error with the Boolean influence of the corresponding coordinates of the target function.
We train our models using $\ell_2$ loss and mini-batch SGD with momentum and batch-size of 64 as the optimizer. Moreover, 
we have repeated each experiment 40 times and averaged the results. 
Furthermore, note that for the MLP model, we pass the Boolean vector directly to the model. However, for the Transformer and MLP-Mixer, we first encode $+1$ and $-1$ tokens into 256-dimensional vectors using a shared embedding layer, and then we pass the embedded input to the models. 
More details on training procedure as well as further experiments are presented in Appendix~\ref{appendix:experiments}.\footnote{Code: \url{https://github.com/aryol/BooleanPVR}} 

\begin{figure}[t]
     \centering
    \includegraphics[width=0.65\textwidth]{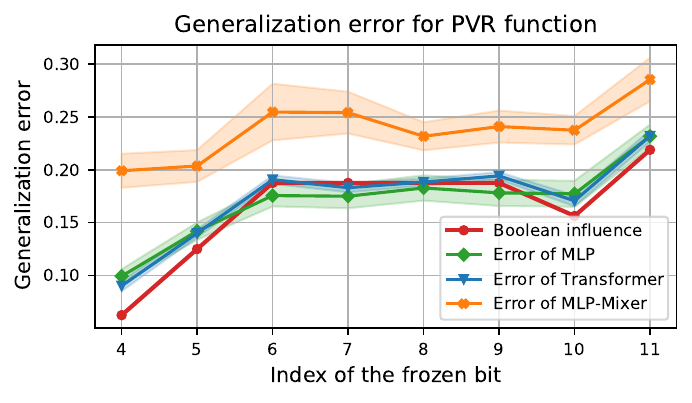}
     \caption{Comparison between the generalization loss in the canonical distribution shift setting and the Boolean influence for a PVR function with 3 pointer bits, window size 3, and majority-vote aggregation function. See Appendix \ref{appendix:experiments} for further experiments. 
     }
     \label{fig:main-exp}
\end{figure}

\paragraph{Influence vs.\ canonical holdout generalization.}
In this section, we consider a Boolean PVR function (as in Def.~\ref{def:PVR}) with 3 pointer bits to be learned by the neural networks. For this function we set window size to 3 and use majority-vote (defined as $g(x_1, \ldots, x_r) = \mathrm{sign}(x_1+\cdots+x_r)$, outputting $-1$, $0$, or $1$) as the aggregation function on the windows.
The generalization error of the models on this PVR function and its comparison with the Boolean influence are presented in Figure~\ref{fig:main-exp}. The x-axis corresponds to the index of the frozen coordinate, that is from 4 to 11 (we do not freeze the pointer bits). On the y-axis, for each frozen coordinate we report the generalization error obtained in the corresponding holdout setting for MLP, the Transformer, and MLP-Mixer, together with the value of the Boolean influence of the frozen coordinate on the target function.
It can be seen that the generalization error of MLP and the Transformer can be well approximated by the Boolean influence. Whereas, the generalization error of MLP-Mixer follows the trend of the Boolean influence with an offset. We remark that in Figure~\ref{fig:main-exp}, the value of the Boolean influence (and gen. error) in the PVR task varies across different indices due to boundary effects and the use of truncated windows (see Def.~\ref{def:PVR}). We refer to Appendix \ref{appendix:experiments} for further experiments on other target functions. 
\paragraph{Implicit bias towards low-degree representation.}
Consider the problem of learning a function $f$ in the canonical holdout setting freezing $x_k=1$. Denote the Fourier coefficients of the frozen function $f_{-k}$ (as defined in Lemma~\ref{lem:Boolean_inf}), by $\hat f_{-k}(S)$, for all $S \subseteq [n]:=\{1,...,n\}$ (recall, $\hat f_{-k}(S) =\E_X[f_{-k}(X) \chi_S(X)] $). For $S$ such that $k \not\in S$, the neural network can learn coefficient $ \hat f_{-k}(S) $ using either $\chi_S(x)$ or $x_k \cdot \chi_S(x) = \chi_{S \cup \{k\}}(x)$ (since these are indistinguishable at training). 
The low-degree implicit bias states that neural networks have a preference for the lower degree monomial $\chi_S$. More precisely, $\chi_S$ is learned faster than $\chi_{S \cup \{k\}}$ and thus the term $\hat f_{-k}(S)\chi_S(x)$ in the Fourier expansion of $f_{-k}$, is mostly learned by the lower degree monomial. Consequently,  according to Lemma \ref{lem:Boolean_inf}, the generalization error will be close to the Boolean influence. Figure~\ref{fig:speed-f1} shows this bias empirically for the above mentioned PVR function and for frozen coordinate $k=6$. 
Figure~\ref{fig:speed-f1} (left) shows that the MLP model has a strong preference for low-degree monomials, Similarly, Figure~\ref{fig:speed-f1} (bottom) shows that the Transformer also captures the monomials in the original function using monomials that exclude the frozen index. Therefore the generalization errors of the MLP and Transformer are very close to the Boolean influence as seen in Figure~\ref{fig:main-exp}. Whereas, Figure~\ref{fig:speed-f1} (right) shows that the MLP-Mixer model has a weaker preference for lower degree monomials (e.g., it learns $1$ and $x_6$ to the same extent) and hence, its generalization error follows the trend of Boolean influence with an offset, which is also presented in Figure~\ref{fig:main-exp}. We refer to Appendix \ref{appendix:experiments} for additional experiments. 

\begin{figure}[t]
\vspace{-2em}
     \centering
    \begin{subfigure}{0.48\textwidth}
         \centering
         \includegraphics[width=\textwidth]{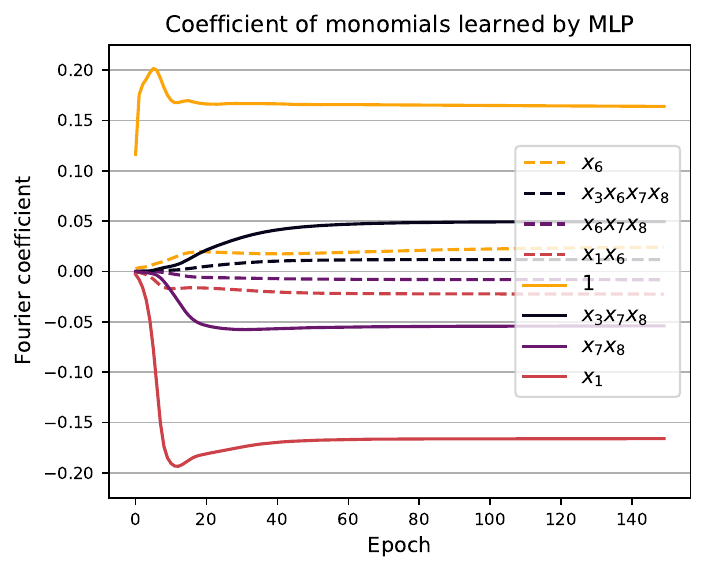}
     \end{subfigure}
     \hfill
     \begin{subfigure}{0.48\textwidth}
         \centering
         \includegraphics[width=\textwidth]{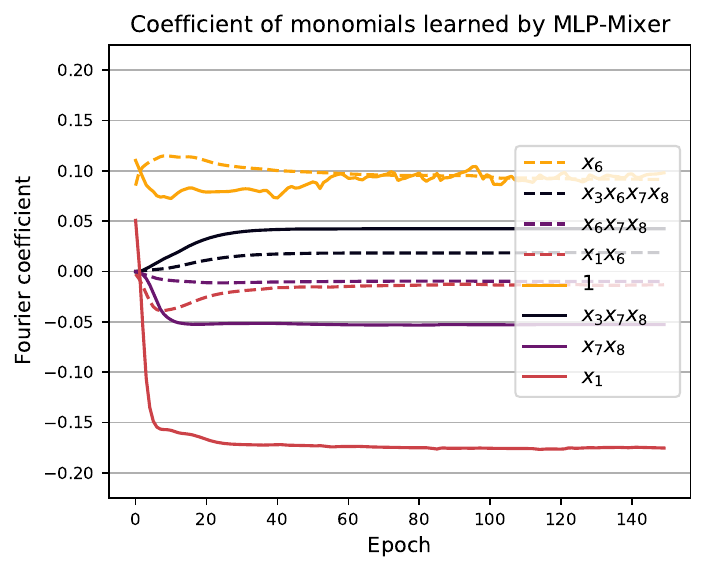}
     \end{subfigure}
    \begin{subfigure}{0.48\textwidth}
         \centering
         \includegraphics[width=\textwidth]{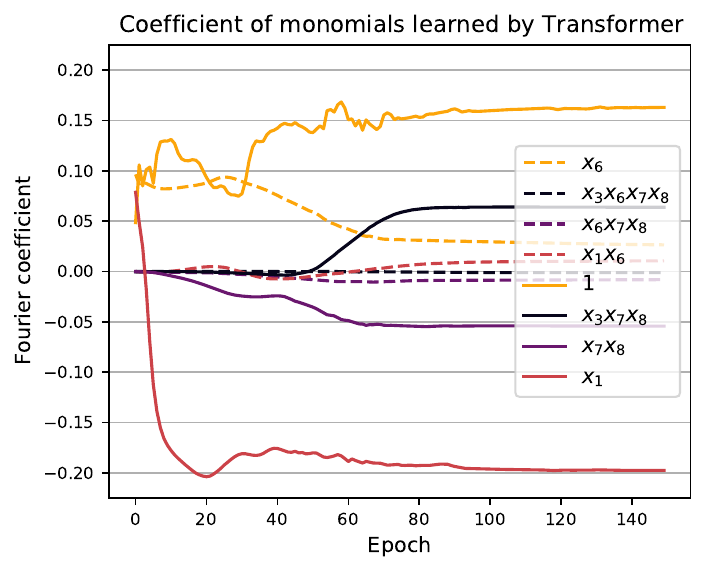}
     \end{subfigure}
     \hfill
        \caption{The coefficients of a selected group of monomials learned by the MLP (left), MLP-Mixer (right) and the Transformer (bottom) when learning the aforementioned PVR function, with $x_{6}=1$ frozen during the training. The coefficient of these monomials in the original function are $\hat f(\{6\}) = 0.1875$, $\hat f(\{3,6,7,8\}) = 0.0625$, $\hat f(\{6,7,8\}) = -0.0625$, and $\hat f(\{1,6\}) = -0.1875$. One can observe that the monomials of the lowest degree are indeed picked up first during the training of the MLP and Transformer, which explains the tight approximation of the Boolean influence for the generalization error in these cases. In contrast, the MLP-Mixer also picks up some contribution from the higher degree monomials including the frozen bit $x_6$.}
        
        \label{fig:speed-f1}
        \vspace{-2.5em}
\end{figure}

\subsection{Linear Models} \label{sec:linearmodels}
In this section, we will focus on linear functions and linear models. First, we state a theorem for linear regression models. Furthermore, we show experiments on the generalization error of linear neural networks and its relation with the depth and initialization of the model. 
\begin{thm} \label{thm:can_dist_shift}
Let $f: \{\pm 1\}^n \to  \mathbb{R}$ be a linear function, i.e., $f(x_1, \cdots, x_n) = \hat f(\emptyset) + \sum_{i=1}^n \hat f(\{i\})x_i$. Consider the canonical holdout where the $k$-th component is frozen at training for a linear regression model where weights and biases are initialized independently with the same mean and variance $\sigma^2$. Also assume the frozen function is unbiased, i.e., $\E_{X_{-k}} [f(X)]$ = 0. In this case, the expected generalization error (over different initializations) of the function learned by GD after $t$ time steps is given by $ \E_{\Theta^0}[\mathrm{gen}(f,\tilde{f}^{(t)}_{-k})]=\mathrm{Inf}_{k}(f) + o_{\sigma^2}(1) + O(e^{-ct}),$
where $c$ is a constant dependent on the learning rate. Moreover, if the frozen function is biased, the expected generalization error is equal to 
 $\E_{\Theta^0}[\mathrm{gen}(f,\tilde{f}^{(t)}_{-k})]=\frac{(\hat f(\emptyset) - \hat f(\{k\}) )^ 2}{4} + o_{\sigma^2}(1) + O(e^{-ct})$.
\end{thm}
\begin{remark}[Kernel regression]
Our result would still hold
if, instead of linear regression, we performed kernel regression with a kernel $K$ that is invertible under the training distribution, i.e., such that $\E_{X,X' \sim \cU_{-k}} [ K(X,X')]$ is invertible. Furthermore, one could try to remove the invertibility assumption by adding a regularization term.
\end{remark}
\begin{remark}[Staircase learning]
An attempt to extend the above result to non-linear function consists of considering staircase functions \cite{abbe2021staircase,mergedstaircase}, and extending \cite{mergedstaircase} to a setting with bias in order to show that SGD learns the lowest-degree representation under proper mean-field initialization of depth-2 neural networks.  
\end{remark}

The proof of Theorem~\ref{thm:can_dist_shift} is presented in Appendix~\ref{sec:proof_mismatch}. In the rest of this section, we empirically show that the condition of zero bias stated in Theorem~\ref{thm:can_dist_shift} is no longer necessary if linear neural networks of large enough depth or small enough initialization are used. 
In fact, the generalization error of linear neural networks makes a transition from the value proved in Theorem~\ref{thm:can_dist_shift} to the Boolean influence as deeper models or smaller scales of initialization are used. We take
$f(x_1, x_2, \ldots, x_{11}) = 1+2x_1-3x_2 +4x_3 -\cdots-11x_{10} +12x_{11}$ as the function that we want to learn. We consider linear neural networks with hidden layers of size 256. Figure~\ref{fig:linear} (left) shows the effect of depth: we initialize weights and biases independently using the uniform distribution $\mathcal{U}(-\frac{1}{\sqrt{N_{in}}} , \frac{1}{\sqrt{N_{in}}})$
where $N_{in}$ is the input dimension of the respective layer. We plot the generalization error in the holdout setting with respect to the corresponding frozen coordinate at training, together with the value of the Boolean influence of each coordinate. We observe that with the increase of depth, the generalization error tends to the Boolean influence. Figure~\ref{fig:linear} (right) show the role of initialization: we take a linear neural network with 3 layers and we initialize weights and biases using the uniform distribution $\mathcal{U}(-N_{in} ^ {-\alpha} , N_{in} ^ {-\alpha})$ with $\alpha$ taking value in $\{0.5, 1, 1.5, 2, 2.5\}$. 
It can be seen that as the initialization scale decreases, the generalization error tends to the Boolean influence.  
\begin{figure}[t]
\vspace{-1em}
     \centering
     \begin{subfigure}[b]{0.48\textwidth}
         \centering
         \includegraphics[width=\textwidth]{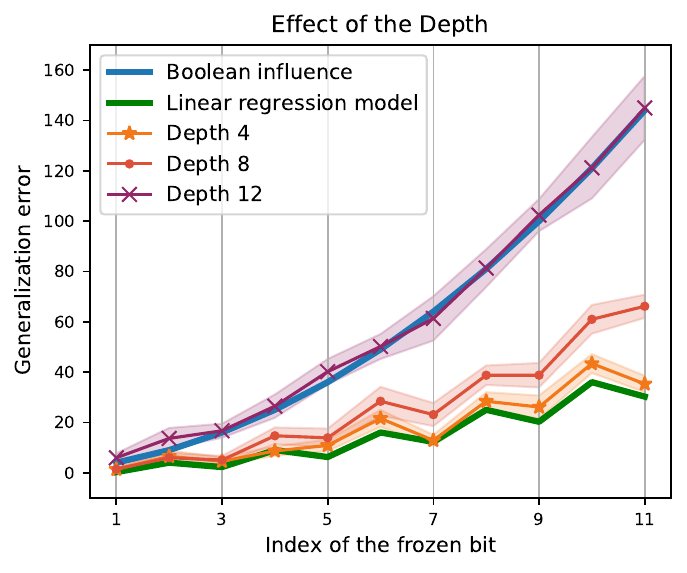}
     \end{subfigure}
     \hfill
     \begin{subfigure}[b]{0.48\textwidth}
         \centering
         \includegraphics[width=\textwidth]{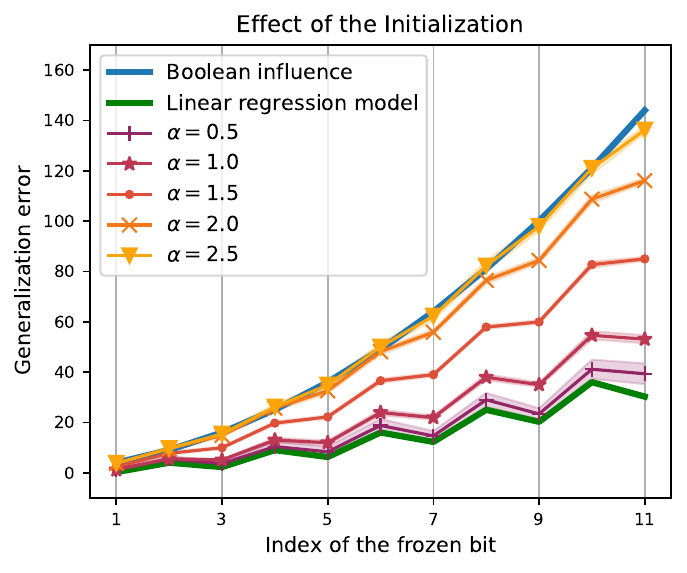}
     \end{subfigure}
     \hfill
        \caption{Effect of depth (left) and initialization scale (right) on the out-of-distribution generalization error of linear neural networks. Generalization error tends to the Boolean influence as depth increases (left) or as the initialization scale decreases (right). 
        }
        \label{fig:linear}
        \vspace{-1em}
\end{figure}

\section{Conclusion}
This paper first establishes a formal result that supports a conjecture made in~\cite{Zhang2021PointerVR}, relating the noise sensitivity of a target function to the generalization error. This gives a first connection between a central measure in Boolean analysis, the noise-sensitivity, and the generalization error when learning Boolean functions with GD. The paper then investigates the generalization error under the canonical holdout. The `low-degree implicit bias hypothesis' is put forward and supported both experimentally and theoretically for certain architectures. This gives a new insight on the implicit bias of GD when training neural networks, that is specific to Boolean functions such as the Boolean PVR, and that relates to the fact that certain networks tend to greedily learn monomials by incrementing their degree. In particular, this allows to characterize the generalization error in terms of the Boolean influence, a second central notion in Boolean Fourier analysis. Boolean measures thus seem to have a role to play in understanding generalization when  learning of `reasoning' or `logical' functions. There are now many directions to pursue such as: (1) extending the realm of architectures/models for which we can prove formally the Boolean influence tightness, (2) considering more general holdout or distribution shift models, (3) investigate how the picture changes when the bits/digits are given by MNIST images, (4) better understanding when the low-degree implicit bias is taking place or not, within and beyond PVR, since the Boolean influence is not always tight in our experiments (e.g., MLP-Mixers seem to have a worse performance than the Boolean influence on PVR; see also Appendix \ref{appendix:experiments} for other functions), (5) investigating how to `revert' the implicit bias towards low-degree when it is taking place, to compensate for the unseen data; this will require justifying and engineering why certain symmetries are favorable in the learned function.

\begin{ack}
We thank Raphaël Berthier (EPFL) and Jan Hązła (EPFL) for useful discussions.
\end{ack}

\bibliography{references}
\bibliographystyle{alpha}
\section*{Checklist}


\begin{enumerate}

\item For all authors...
\begin{enumerate}
  \item Do the main claims made in the abstract and introduction accurately reflect the paper's contributions and scope?
    \answerYes{}
  \item Did you describe the limitations of your work?
    \answerYes{}
  \item Did you discuss any potential negative societal impacts of your work?
    \answerNA{We believe there are no potential negative societal impacts.}
  \item Have you read the ethics review guidelines and ensured that your paper conforms to them?
    \answerYes{}
\end{enumerate}

\item If you are including theoretical results...
\begin{enumerate}
  \item Did you state the full set of assumptions of all theoretical results?
    \answerYes{}
        \item Did you include complete proofs of all theoretical results?
    \answerYes{}
\end{enumerate}

\item If you ran experiments...
\begin{enumerate}
  \item Did you include the code, data, and instructions needed to reproduce the main experimental results (either in the supplemental material or as a URL)?
    \answerYes{Link to the GitHub repository is provided.}
  \item Did you specify all the training details (e.g., data splits, hyperparameters, how they were chosen)?
    \answerYes{See Section~\ref{sec:exps} and Appendix on experiment details.}
        \item Did you report error bars (e.g., with respect to the random seed after running experiments multiple times)?
    \answerYes{The $95\%$ confidence interval is reported if relevant.}
        \item Did you include the total amount of compute and the type of resources used (e.g., type of GPUs, internal cluster, or cloud provider)?
    \answerYes{See Appendix on experiment details}
\end{enumerate}

\item If you are using existing assets (e.g., code, data, models) or curating/releasing new assets...
\begin{enumerate}
  \item If your work uses existing assets, did you cite the creators?
    \answerYes{}
  \item Did you mention the license of the assets?
    \answerYes{}
  \item Did you include any new assets either in the supplemental material or as a URL?
    \answerYes{}
  \item Did you discuss whether and how consent was obtained from people whose data you're using/curating?
    \answerNA{}
  \item Did you discuss whether the data you are using/curating contains personally identifiable information or offensive content?
    \answerNA{}
\end{enumerate}

\item If you used crowdsourcing or conducted research with human subjects...
\begin{enumerate}
  \item Did you include the full text of instructions given to participants and screenshots, if applicable?
    \answerNA{}
  \item Did you describe any potential participant risks, with links to Institutional Review Board (IRB) approvals, if applicable?
    \answerNA{}
  \item Did you include the estimated hourly wage paid to participants and the total amount spent on participant compensation?
    \answerNA{}
\end{enumerate}

\end{enumerate}



\newpage
\appendix

\section{Proof of Theorem~\ref{thm:distr_match}} \label{sec:proof_distr_match}
The proof of Theorem~\ref{thm:distr_match} goes by two steps. As a first step, we connect the noise stability to another measure of complexity for classes of functions called cross predictability (CP). As a second step, we use the negative result from~\cite{AS20}, that lower bounds the generalization error of learning a class of functions in terms of its cross predictability. 

\subsection{From noise stability to cross predictability}
We redefine here the cross predictability (CP), for completeness.
\begin{defin}[Cross Predictability~\cite{AS20}] \label{def:CP}
Let $\cX$ be the input space and let $\cF$ be a class of functions. Let $P_{\cX}$ and $P_{\cF}$ be two distributions supported on $\cX$ and $\cF$ respectively. Their cross-predictability is defined as
\begin{align}
\CP(P_\cF,P_\cX) = \E_{F,F'\sim P_{\cF} } [\E_{X \sim P_{\cX}} [ F(X) \cdot  F'(X) ]^2 ].
\end{align}
\end{defin}
\noindent
Before diving into the proof we give few definitions that will be useful. 
Given a target function $f$, we define the ``orbit'' of $f$ ($\orb(f)$) as the class of all functions generated by composing $f$ with a permutation of the input space: 
\begin{defin}[Orbit]
For $f : \bR^n \to \bR$ and a permutation $\pi \in S_n$, we let $(f \circ \pi)(x) = f(x_{\pi(1)},...,x_{\pi(n)})$. Then, the \emph{orbit} of $f$ is defined as
\begin{align}
\orb(f) : = \{ f \circ \pi : \pi \in S_n  \} . 
\end{align}
\end{defin}
GD (or SGD) on a neural network with initialization that is target agnostic has equivalent behaviour when learning any target function in $\orb(f)$. We believe that one could extend the result to other invariances, beyond permutations.

Recall from Definition~\ref{def:Nextension} that we introduced an augmented input space, to guarantee that the high-degree Fourier coefficients of the target function are sparse enough. Thus, let 
$\bar  f:\{\pm 1\}^{2n} \to \{\pm 1\}  $ be the $2n-$extension of $f$, defined as 
$\bar{f} (x_1,...,x_n,x_{n+1},...,x_{2n}) = f(x_1,...,x_n).$
For brevity we make use of the following notation:
\begin{align}
    \CP(\orb(\bar f)) := \CP(\cU_{\orb(\bar f)},\cU_{\bF_{2}^{2n}}),
\end{align}
where $\cU_{\orb(\bar f)},\cU_{\bF_{2}^{2n}} $ denote the uniform distribution over $\orb(\bar f) $ and over $ \bF_{2}^{2n}$ (i.e., the $2n$-dimensional Boolean hypercube), respectively.

Furthermore, recall that every Boolean function $f$ can be written in terms of its Fourier-Walsh expansion $ f(x) = \sum_{S} \hat f(S)  \chi_S(x)$, where $\chi_S(x) = \prod_{i \in S} x_i$ are the standard Fourier basis elements and $\hat f(S) $ are the Fourier coefficients of $f$. We further denote by
\begin{align}
    W^{k} (f)  = \sum_{S: |S| = k} \hat f(S)^2 \quad \text{ and } \quad 
    W^{\leq k} (f)  = \sum_{S: |S| \leq k} \hat f(S)^2,
\end{align}
the total weight of the Fourier coefficients of $f$ at degree $k$ and up to degree $k$, respectively.\\
Let $\hat f$ be the Fourier coefficients of the original function $f$, and let $\hat h$ be the coefficients of the augmented function $\bar f$, that are:
\begin{align} 
    \hat  h(T) &= \hat f(T)  
    & \hspace{-3cm}\text{ if } T \subseteq [n] \label{eq:fourier_fbar_1}\\
    \hat h(T) & = 0  & \hspace{-3cm}\text{ otherwise}. \label{eq:fourier_fbar_2}
\end{align}
\noindent
We make use of the following Lemma, that relates the cross-predictability of $\orb(\bar f)$ to the Stability of $f$.
\begin{lemma} \label{lem:CP_Stab}
There exists $\delta$ such that for any $\delta' < \delta $
\begin{align}
    \CP(\orb(\bar f)) \leq  \Stab_{\delta'} ( f).
\end{align}
\end{lemma}

\begin{remark}[Noise Stability]
We remark the following two properties of $\Stab_\delta[f]$:
\begin{enumerate}
    \item One can show (see e.g., Theorem 2.49 in~\cite{o'donnell_2014}) that 
\begin{align}
    \Stab_\delta(f) = \sum_{k=1}^n (1-2\delta)^k W^{k}[f],
\end{align}
where $ W^{k}[f]$ is the Boolean weight at degree $k$ of $f$;
    \item For all $\delta \in [0, 1/2]$, $\Stab_{\delta} ( f) = \Stab_{\delta } (\bar f) $. This follows directly from the previous point and~\eqref{eq:fourier_fbar_1}-\eqref{eq:fourier_fbar_2}.
\end{enumerate}

\end{remark}
\noindent
\begin{proof}[Proof of Lemma~\ref{lem:CP_Stab}]
We denote by $\pi$ a random permutation of $2n$ elements. We can bound the $\CP(\orb(\bar f))$ by the following:
\begin{align}
    \CP(\orb(\bar f)) & = \E_{\pi}\left[ \E_X \left[ \bar f(X) \bar f(\pi(X))\right]^2\right]\\
    & = \E_{\pi}\left( \sum_{T\subseteq [2n]} \hat h(T) \hat h  (\pi(T)) \right)^2 \label{eq:parseval} \\
    & = \E_{\pi}\left( \sum_{T \subseteq [n]} \hat f(T) \hat f (\pi(T))  \cdot \mathds{1}\left(\pi(T) \subseteq [n]\right) \right)^2  \label{eq:coefficients}\\
    & \overset{C.S}{\leq} \E_{\pi} \left(\sum_{S\subseteq[n]} \hat f(\pi(S))^2  \right) \cdot \left( \sum_{T\subseteq[n]} \hat f(T)^2 \mathds{1}\left(\pi(T) \subseteq [n]\right) \right) \label{eq:cs1}\\
    & = \sum_{T \subseteq [n]} \hat f(T)^2 \cdot  \pr_\pi \left(\pi(T) \subseteq [n]\right) \label{eq:Bool_sum}\\
    &= \sum_{k=1}^n W^k[f] \cdot \pr_\pi \left(\pi(T) \subseteq [n] \mid |T| = k\right), \label{eq:WkProb}
\end{align}
where~\eqref{eq:parseval} is the scalar product in the Fourier basis,~\eqref{eq:coefficients} follows by applying the formulas of the $\hat h$ given in \eqref{eq:fourier_fbar_1}-\eqref{eq:fourier_fbar_2},~\eqref{eq:cs1} holds by Cauchy-Schwarz inequality,~\eqref{eq:Bool_sum} holds since $f $ is Boolean-valued and for each $\pi$ by Parseval identity, $\sum_{S\subseteq[n]} \hat f(\pi(S))^2  = \E_X[f(X)^2] =1 $, and~\eqref{eq:WkProb} holds since the second term is invariant for all sets of a given cardinality.
\\
Recalling $\pi$ is a random permutation over the augmented input space of dimension $2n$, for each $k \in [n]$ we can further bound the second term by 
\begin{align}
    \pr_\pi \left(\pi(T) \subseteq [n] \mid |T| = k \right) & = \frac{{n \choose k}}{{2n \choose k}} \sim \frac{1}{2^k} \leq (1-2\delta')^k, \quad \text{ for all $\delta'\leq 1/4$.}
\end{align}
 Thus, for all $\delta'\leq 1/4$,
\begin{align}
    \CP(\orb(\bar f)) & \leq \sum_{k=1}^n  (1-2\delta')^k W^k[f] = \Stab_{\delta'}[f].
\end{align}
\end{proof}
\begin{remark}
Note that the value of $\delta$ in Lemma~\ref{lem:CP_Stab} depends on the size of the input extension that we use. In this paper, we defined an input extension of size $2n$ (input doubling), which gives $\delta = 1/4$, however we could have chosen e.g. a $3n$-extension and obtain $\delta = 1/3$, and so on.
\end{remark}

\subsection{From cross predictability to hardness of learning}
For the second step, we make use of Theorem 3 and Corollary 1 in~\cite{AS20}, that prove a lower bound of learning a class of function in terms of its cross predictability. 
The lower bound holds for the noisy GD algorithm (\cite{AS20,abbe2021power}), of which we give a formal definition here.
\begin{defin}[Noisy GD with batches] \label{def:noisyGD}
Consider a neural network of size $E$, with a differentiable non-linearity and initialization of the weights $W^{(0)}$. Given a differentiable loss function, the updates of the noisy GD algorithm with learning rate $\gamma_t$ and gradient precision $A$ are defined by
\begin{align}
    W^{(t)} = W^{(t-1)}- \gamma_t \E_{X \sim S^{(t)}} [ \nabla L(f(X),f_{NN}^{(t)}) ]_A + Z^{(t)}, \qquad t = 1,...,T,
\end{align}
where for all $t$, $Z^{(t)}$ are i.i.d. $\cN(0, \sigma^2)$, for some $\sigma$, and they are independent from other variables, $S^{(t)}=(X_1^{(t)},...,X_m^{(t)})$ has independent components drawn from the input distribution $P_\cX$ and independent from other time steps, and $f$ is the target function, from which the labels are generated, and by $[.]_A$ we mean that whenever the argument is exceeding $A$ (resp. $-A$) it is rounded to $A$ (resp. $-A$).
\end{defin}
Theorem 3 and Corollary 1 in~\cite{AS20} imply that for any distribution over the Boolean hypercube $P_\cX$ and Boolean functions $P_\cF$, it holds that
\begin{align}
     \pr_{X,F\sim P_\cF, f_{NN}^{(T)}}(F(X) \neq f_{NN}^{(T)}(X)) \geq 1/2 - \frac{\gamma T \sqrt{E} A }{\sigma} \left( 1/m + \CP(P_\cF,P_\cX) \right)^{1/4},
\end{align}
where $\gamma,E, A, \sigma, m$ have the same meaning as in Definition~\ref{def:noisyGD}. As observed by them, in our case since the initialization is invariant under permutations of the input, then learning the orbit of $\bar f$ under uniform distribution is equivalent to learning $\bar f$, thus the following bound holds:
\begin{align}
     \pr_{X, f_{NN}^{(T)}}(\bar f(X) \neq f_{NN}^{(T)}(X)) \geq 1/2 - \frac{\gamma T \sqrt{E} A }{\sigma} \left( 1/m + \CP(\orb(\bar f)) \right)^{1/4}.
\end{align}



\section{Removing the input doubling} \label{sec:no_input_augmentation}
One can prove a similar result to the one of Theorem~\ref{thm:distr_match}, without using the input extension technique. However, we need some additional assumptions on $f$. To introduce them, let us first fix some notation.
In the following, we say that a sequence $a_n$ is
\emph{noticeable} if there exists $c\in\mathbb{N}$ such that $a_n=\Omega(n^{-c})$.
On the other hand, we say that $f$ is \emph{negligible} if $\lim_{n\to\infty} n^c a_n=0$ for
every $c\in\mathbb{N}$ 
(
which we also write $a_n = n^{-\omega(1)}$).
\begin{assumption}[Non-dense and non-extremal function] \label{ass:no_doubling}\text{}
\begin{itemize}
\item[a)] We say that $f$ is ``non-dense'' if there exists $c$ such that $W\{ T: \hat f(T)^2 \leq n^{-c} \} = n^{-\omega(1)}$, i.e., the negligible Fourier coefficients do not bring a noticeable contribution if taken all together;
\item[b)] We say $f$ is ``non-extremal'' if for any positive constant $D$, $W^{\geq n-D}[f] = n^{-\omega(1)}$, i.e., $f$ does not have noticeable Fourier weight on terms of degree $n-O(1)$.
\end{itemize}
\end{assumption}
\noindent
With such additional assumptions, we can conclude the following.
\begin{prop} \label{prop:no_doubling}
Let $ f:\{\pm 1\}^n \to \{ \pm 1\}$ be a balanced target function, let $\Stab_\delta(f)  $ be its noise stability and let $f^{(t)}_{NN}$ be the output of $GD$ with gradient precision $A$ after $t$ time steps, trained on a neural network of size $E$ with initialization that is target agnostic. Assume $f$ satisfies Assumption~\ref{ass:no_doubling}. Then, there exist $c,C>0$ and $D>0$ such that if $\delta < D/n$
\begin{align}
    {\rm gen}(f, f_{NN}^{(t)}) \geq 1/2-C \cdot t\cdot \sqrt{E} \cdot \left(  n^c \cdot \Stab_\delta(f) + n^{-\omega(1)} \right)^{1/4}.
\end{align}
\end{prop}
\noindent
The proof of Proposition~\ref{prop:no_doubling} resembles the proof of Theorem~\ref{thm:distr_match}. The only modification required is in Lemma~\ref{lem:CP_Stab}, which is replaced by the following Lemma.

\begin{lemma} \label{lem:CP_stab_no_doubling}
Let $f$ be a Boolean function that satisfies Assumption~\ref{ass:no_doubling}. There exists $c,D>0$ such that for $\delta<D/n$,
\begin{align}
    \CP(\orb(f)) \leq 2 \cdot n^c \cdot \Stab_\delta(f) + n^{-\omega(1)}.
\end{align}
\end{lemma}
\begin{proof}[Proof of Lemma~\ref{lem:CP_stab_no_doubling}]
Let $c>0$ be such that $W\{ T: \hat f(T)^2 \leq n^{-c} \} = n^{-\omega(1)}$. This $c$ exists because of Assumption 1a. 
\begin{align}
    \CP &(\orb(f)) =   \\
    &= \E_{\pi}\left[ \E_X \left[ f(X)  f(\pi(X))\right]^2\right]\\
    & = \E_{\pi}\left( \sum_{T\subseteq [n]} \hat f(T) \hat f  (\pi(T)) \right)^2 \\
    & = \E_{\pi}\left( \sum_{T \subseteq [n]} \hat f(T) \hat f (\pi(T))  \cdot \left( \mathds{1}\left(\hat f(\pi(T))^2 \leq n^{-c} \right) + \mathds{1}\left(\hat f(\pi(T))^2 > n^{-c} \right)  \right) \right)^2 \\
    & \leq 2 \E_{\pi}\left( \sum_{T \subseteq [n]} \hat f(T) \hat f (\pi(T)) \mathds{1}\left(\hat f(\pi(T))^2 \leq n^{-c} \right) \right)^2 +\\
    & \hspace{3cm}+  2 \E_{\pi} \left(\sum_{T \subseteq [n]} \hat f(T) \hat f (\pi(T)) \mathds{1}\left(\hat f(\pi(T))^2 > n^{-c} \right) 
    \right)^2 \label{eq:split_nc},
\end{align}
where in the last inequality we used $(a+b)^2 \leq 2(a^2+b^2)$.
Let us first focus on the second term on the right.
\begin{align}
    \E_{\pi}&\left( \sum_{T \subseteq [n]}\hat f(T) \hat f (\pi(T)) \mathds{1}\left(\hat f(\pi(T))^2 > n^{-c} \right) \right)^2 \\
    &\overset{C.S}{\leq} \E_{\pi} \left(\sum_{S\subseteq[n]} \hat f(\pi(S))^2  \right) \cdot \left( \sum_{T\subseteq[n]} \hat f(T)^2 \mathds{1}\left(\hat f(\pi(T))^2 > n^{-c}\right) \right)\\
    & \leq \sum_{T \subseteq [n]} \hat f(T)^2 \cdot  \pr_\pi \left(\hat f(\pi(T))^2 > n^{-c}\right)\\
    &= \sum_{k=1}^n W^k[f] \cdot \pr_\pi \left(\hat f(\pi(T))^2 > n^{-c} \mid |T| = k\right)\\
    & = \sum_{k=1}^{n-D} W^k[f] \cdot \pr_\pi \left(\hat f(\pi(T))^2 > n^{-c} \mid |T| = k\right) +\\
    & \hspace{3cm} + \sum_{k=n-D+1}^{n} W^k[f] \cdot \pr_\pi \left(\hat f(\pi(T))^2 > n^{-c} \mid |T| = k\right)\nonumber\\
    & \leq \sum_{k=1}^{n-D} W^k[f] \cdot \pr_\pi \left(\hat f(\pi(T))^2 > n^{-c} \mid |T| = k\right) + W^{\geq n-D+1}[f].
\end{align}
where $D$ is an arbitrary positive constant. Because of Assumption 1b, $W^{\geq n-D+1}[f] = n^{-\omega(1)} $. On the other hand, since $f$ is a Boolean valued function,
\begin{align}
    \sum_{T}\hat f(T)^2 = \E_X[f(X)^2] =1,
\end{align}
which implies that there are at most $n^c$ sets $T$ such that $ \hat f(T)^2 >n^{-c}$. Thus, recalling $\pi$ is a random permutation over the input space of dimension $n$, we get
\begin{align}
     \pr_\pi \left(\hat f(\pi(T))^2 > n^{-c} \mid |T| = k\right) & \leq \frac{n^c}{{n \choose k} } \\
     & \leq n^c \left( \frac{k}{n}  \right)^k \label{eq:binomial_bound}\\
     & \leq n^c \left( \frac{n-D}{n}\right)^k\\
     & \leq n^c \left( 1-2\delta\right)^k \qquad \text{ if   }\delta \leq \frac{D}{2n},
\end{align}
where in~\eqref{eq:binomial_bound} we used that ${n \choose k} \geq (\frac{n}{k})^k$ for all $k \geq 1$.
Going back to the first term in~\eqref{eq:split_nc} we get
\begin{align}
    &\E_{\pi}\left( \sum_{T \subseteq [n]} \hat f(T) \hat f (\pi(T)) \mathds{1}\left(\hat f(\pi(T))^2 \leq n^{-c} \right) \right)^2 \\
    & \overset{C.S}{\leq}  \E_{\pi} \left(\sum_{S\subseteq[n]} \hat f(S)^2  \right) \cdot \left( \sum_{T\subseteq[n]} \hat f(\pi(T))^2 \mathds{1}\left(\hat f(\pi(T))^2 > n^{-c}\right) \right)\\
    & \leq \sum_{T\subseteq[n]} \hat f(\pi(T))^2 \mathds{1}\left(\hat f(\pi(T))^2 > n^{-c}\right)\\
    & = n^{-\omega(1)},
\end{align}
by Assumption 1a. Hence overall,
\begin{align}
    \CP (\orb(f)) &\leq 2 n^c \sum_{k=1}^{n-D} W^{k}[f] (1-2\delta)^k +n^{-\omega(1)}\\
    &  \leq 2 n^c \Stab_{\delta}(f) +n^{-\omega(1)}.
\end{align}
\end{proof}

\section{Proof for Lemma \ref{lem:Boolean_inf} and Theorem \ref{thm:can_dist_shift}}\label{sec:proof_mismatch}
In this section, we present proofs for results mentioned in Section \ref{sec:mismatch_setting}, namely, Lemma~\ref{lem:Boolean_inf} and Theorem~\ref{thm:can_dist_shift}.

\subsection{Proof of Lemma \ref{lem:Boolean_inf}}
\begin{proof}[Proof of Lemma~\ref{lem:Boolean_inf}]
Let $f(x) = \sum_{T \subseteq [n]} \hat{f}(T)\chi_T(x)$ be the Fourier expansion of the function where $\chi_T(x) = \prod_{i \in T} x_i$. Therefore, the Fourier expansion of the frozen function will become
\begin{equation}
    f_{-k}(x) = \sum_{T \subseteq [n]\setminus k} (\hat{f}(T) + \hat{f}(T\cup k))\chi_T(x).
\end{equation} 
Thus, the difference between functions is equal to 
\begin{equation}
    (f-f_{-k})(x) = \sum_{T \subseteq [n]:k \in T} \hat{f}(T)\chi_T(x) - \sum_{T \subseteq [n]\setminus k} \hat{f}(T\cup k)\chi_T(x).
\end{equation} 
Hence, using Parseval's Theorem we have the following:
\begin{equation}
    \mathbb{E}_{U^n} (f-f_{-k})_2^2 = \sum_{T \subseteq [n]:k \in T} \hat{f}(T)^2 + \sum_{T \subseteq [n]\setminus k} \hat{f}(T\cup k)^2 = 2\sum_{T \subseteq [n]:k \in T} \hat{f}(T)^2.
\end{equation} 
Therefore,
 \begin{equation}
     \mathbb{E}_{U^n} \frac{1}{2} (f-f_{-k})_2^2 = \sum_{T \subseteq [n] : k \in T} \hat{f}(T)^2 = \mathrm{Inf}_{k}(f), 
 \end{equation}
 and the lemma is proved. 
\end{proof}
\subsection{Proof of Theorem \ref{thm:can_dist_shift}}
\begin{proof}[Proof of Theorem~\ref{thm:can_dist_shift}] Assume $\tilde{f}_{-k}^{(t)}(x,\Theta^{(t)}):= x^{T}W^{(t)}+b^{(t)}$ to be our linear model where $\Theta^{(t)} = (W^{(t)},b^{(t)})$ are the model parameters at time $t$. In the following, the super-script $t$ and $T$ denote the time-step and transpose respectively. Also, we use $\E_{x_{-k}}$ to denote the expectation of $x$ taken uniformly on the Boolean hypercube while $x_k = 1$.
Using the square loss, we have
\begin{align}
    L(\Theta^{(t)}, x,f) = (x^T W^{(t)} +b^{(t)} - f(x))^2,
\end{align}
and the gradients will be
\begin{align}
    \nabla_W L (\Theta^{(t)}, x,f) &= 2 x \left( x^T W^{(t)} +b^{(t)} - f(x)\right),\\
    \partial_b L (\Theta^{(t)}, x,f) &= 2 \left( x^T W^{(t)} +b^{(t)} - f(x)\right).
\end{align}
The GD update rule will then become
\begin{align}
    W^{(t+1)} &= W^{(t)} - 2 \gamma \left( \E_{x_{-k}} \left[ x x^T \right] W^{(t)} +  \E_{x_{-k}}[x] b^{(t)} - \E_{x_{-k}}[x f (x)]\right),\\
    b^{(t+1)} & = b^{(t)} - 2 \gamma \left( \E_{x_{-k}} [x^T] W^{(t)} + b^{(t)} - \E_{x_{-k}} [f(x)] \right).
\end{align}
Note that $\E_{x_{-k}} \left[ x x^T \right] = \mathbb{I}_n $, $\E_{x_{-k}}[x] = \vec e_{k}$. So we have
\begin{align}
    \forall j\neq k:~W_{j}^{(t+1)} &= W_{j}^{(t)}(1 - 2 \gamma) + 2 \gamma \E_{x_{-k}} [x_j \cdot  f(x)],\\
    W_{k}^{(t+1)} &= W_{k}^{(t)} - 2 \gamma ( W_{k}^{(t)} +b^{(t)}) + 2\gamma \E_{x_{-k}} [f(x)],\\
    b^{(t+1)} & = b^{(t)}  - 2 \gamma (W_{k}^{(t)}+b^{(t)}) + 2 \gamma \E_{x_{-k}} [f(x)].
\end{align}
Using above equations, we have
\begin{align}
   W_{k}^{(t+1)} -  b^{(t+1)} &= W_{k}^{(t)} -  b^{(t)} = W_{k}^{(0)} -  b^{(0)}, \\
    W_{k}^{(t+1)}  +  b^{(t+1)} &= (1-4\gamma)(W_{k}^{(t)} +  b^{(t)}) + 4 \gamma \E_{x_{-k}} [f(x)].
\end{align}
Assume $\gamma < \frac{1}{4}$ and define $ 0< c =-\log(1-2\gamma) <-\log(1-4\gamma)$, then we have
\begin{align}
        W_{k}^{(t)}+ b^{(t)} &= (1-4\gamma)^t(W_{k}^{(0)} + b^{(0)} - \E_{x_{-k}} [f(x)]) + \E_{x_{-k}} [f(x)] \nonumber\\&= O((1-4\gamma)^t) + \E_{x_{-k}} [f(x)] = O(e^{-ct})+ \E_{x_{-k}} [f(x)]  \nonumber\\
        &= O(e^{-ct}) + \hat f(\emptyset) + \hat f(\{k\}),\\ 
       \forall j\neq k:~ W^{(t)}_{j} &= (1-2\gamma)^t(W^{(0)}_{j} - \E_{x_{-k}} [x_j \cdot f(x)]) + \E_{x_{-k}} [x_j \cdot f(x)] \nonumber\\&= O((1-2\gamma)^t)+ \E_{x_{-k}} [x_j \cdot f(x)] = O(e^{-ct}) + \E_{x_{-k}} [x_j \cdot f(x)] \nonumber\\
       &= O(e^{-ct}) + \hat f (\{j\}).
\end{align}
So the learned function is
\begin{align}
    \tilde{f}_{-k} (x;\Theta^{(t)}) = \frac{b^{(0)} - W_k^{(0)} + \hat f(\emptyset) + \hat f(\{k\})}{2} &+ \frac{W_k^{(0)}-b^{(0)}  + \hat f(\emptyset) + \hat f(\{k\})}{2}x_k\nonumber\\
    &
     + \sum_{j \neq k} \hat f(\{j\})\cdot  x_j + O(e^{-ct})
\end{align}
and the generalization error can be computed using Parseval Theorem:
\begin{align}
    &\mathrm{gen}(f,\tilde{f}_{-k}^{(t)}) = \frac 12\E_{x \sim U^n} \left[\left(f(x) - \tilde{f}_{-k}^{(t)}(x;\Theta^\infty) \right)^2 \right] \\ 
    &=\frac 12 \left(\frac{(b^{(0)} - W_k^{(0)} - \hat f(\emptyset) + \hat f(\{k\}))^2 + (W_k^{(0)}-b^{(0)}  + \hat f(\emptyset) - \hat f(\{k\}))^2}{4}\right)   + O(e^{-ct}) \\
    &=\frac{(b^{(0)} - W_k^{(0)} - \hat f(\emptyset) + \hat f(\{k\}))^2}{4}
    + O(e^{-ct}) \\
    &=
    \frac{(b^{(0)} - W_k^{(0)})^2}{4} +
    \frac{(\hat f(\emptyset) - \hat f(\{k\}))^2}{4}
    -2\frac{(b^{(0)} - W_k^{(0)})(\hat f(\emptyset) - \hat f(\{k\})}{4}
    + O(e^{-ct}).
\end{align}
Therefore, the expected generalization loss over different initializations is given by
\begin{align}
    \E_{\Theta^0}[\mathrm{gen}(f,\tilde{f}^{(t)}_{-k})] &= \E_{\Theta^0} \left[ \frac{(b^{(0)} - W_k^{(0)})^2 +
    (\hat f(\emptyset) - \hat f(\{k\}))^2}{4}
    \right]+ O(e^{-ct}) \\&= \frac{
    (\hat f(\emptyset) - \hat f(\{k\}))^2}{4} + \frac{\sigma^2}{2} + O(e^{-ct}).
\end{align}
Particularly, if the frozen function is unbiased, i.e., $\hat f(\emptyset) + \hat f(\{k\}) = 0$, we have
\begin{align}
    \E_{\Theta^0}[\mathrm{gen}(f,\tilde{f}^{(t)}_{-k})] &= \frac{
    (2\hat f(\{k\}))^2}{4} + \frac{\sigma^2}{2} + O(e^{-ct}) \nonumber\\
    &= 
    \hat f(\{k\})^2 + \frac{\sigma^2}{2} + O(e^{-ct}) = \mathrm{Inf}_k(f) + \frac{\sigma^2}{2} + O(e^{-ct}).
\end{align}
\end{proof}
\section{Further details on noise stability}
\label{sec:NSforPVR}
\subsection{Noise stability of PVR functions} 
As mentioned above, a PVR function consists of a pointer (the first bits of the input) and an aggregation function that acts on a specific window indicated by the pointer. We denote by $p$ the number of bits that define the pointer, and by $w$ the size of each window. For simplicity, we consider a slight variation of Boolean PVR task with non-overlapping windows, defined as follows:
\begin{itemize}
    \item PVR with \emph{non-overlapping} windows: the $2^p$ windows pointed by the pointer bits are non-overlapping, i.e., the first window is formed by bits $x_{p+1},..., x_{p+w}$, the second window is formed by bits $x_{p+w+1},...,x_{p+2w}$, and so forth.
\end{itemize}
The input size is thus given by $n:= p+2^p w$ and $p = O(\log(n))$. We denote by $g: \{\pm 1 \}^w \to \{ \pm 1\}$ the aggregation function, which we assume to be balanced (i.e., $\E_X[g(X)] = 0$). One can verify (see details below) that the noise stability of the PVR function $f$ is given by 
\begin{align} \label{eq:stabPVR}
    \Stab_\delta[f] = (1-\delta)^{p+w} + (1-\delta)^p (1-(1-\delta)^w) \cdot \Stab_\delta[g].
\end{align}
We notice that the $\Stab_{\delta}[f]$ is given by two terms: the first one depends on the window size and the second one on the stability of the aggregation function. For large enough window size, the second term in~\eqref{eq:stabPVR} is the dominant one, and $\Stab_\delta[f] $ depends on the stability of $g$. Thus from Theorem~\ref{thm:distr_match}, $f$ is not learned by GD (in the extended input space) in poly(n) time if the stability of the aggregation function is $n^{-\omega(1)}$. 
On the other hand, for small window size (specifically for $w = O(\log(n))$), the $\Stab_\delta(f)$ is `noticeable' (as defined in Appendix~\ref{sec:no_input_augmentation}) for every aggregation function, since the function value itself depends on a limited number of input bits. Thus, noise unstable aggregation functions (e.g. parities) can form a PVR function with `noticeable' stability, if the window size is $O(\log(n))$.
As examples, we consider the specific cases of pairty and majority vote as aggregation functions. 
\begin{itemize}
    \item Parity: If we choose $g(x_1,...,x_w) = \prod_{i=1}^w x_i$, one can observe that $\Stab_\delta(g) = (1-2\delta)^w$. Then, eq.~\eqref{eq:stabPVR} becomes $\Stab_\delta(f) = (1-\delta)^{w+p} [ 1-(1-2\delta)^w] +(1-\delta)^p $, and $\Stab_\delta(f)$ is decreasing with $w$.
    \item Majority: If we choose $g$ to be $g(x_1,...,x_w) = \sign(\sum_{i=1}^w x_i)$, then, for $w$ large, $\Stab_\delta (g) \sim 1-2/\pi \cdot  \arccos(1-2\delta)$ (see e.g.~\cite{o'donnell_2014}). Plugging this in eq.~\eqref{eq:stabPVR}, one can observe that also for majority vote $\Stab_\delta(f) $ is decreasing with $w$.
\end{itemize}

\paragraph{Computation of~\eqref{eq:stabPVR}.}
We compute the expression in~\eqref{eq:stabPVR} with the following: 
\begin{align}
    \Stab_\delta[f] &= 1 - 2 \NS_\delta[f],
\end{align}
where $\NS_\delta[f] := \pr(f(X) \neq f(Y))$ is the Noise sensitivity of $f$, defined as the probability that perturbing each input bit independently with probability $\delta$ changes the output of $f$ and where we denoted by $Y $ the vector obtained from $X$ by flipping each component with prob. $\delta$ independently. To compute $ \NS_\delta[f]$, we can first distinguish depending on whether the perturbation affects the pointer bit:
\begin{align*}
    \NS_\delta[f] :&= \pr(f(X) \neq f(Y)) \\
    &  = (1-\delta)^p \cdot \pr(f(X) \neq f(Y) \mid X^p = Y^p )  + (1-(1-\delta)^p) \pr(f(X) \neq f(Y) \mid X^p \neq Y^p ) \\
    & = (1-\delta)^p \cdot \pr(f(X) \neq f(Y) \mid X^p = Y^p )  + (1-(1-\delta)^p) \frac 12,
\end{align*}
where the last inequality holds since we are using non-overlapping windows and we assumed $g$ to be  balanced. To compute the first term, we can condition on whether any bit in the window pointed by $X$ and $Y$ is changed:
\begin{align*}
    \pr(f(X) \neq &f(Y) \mid X^p = Y^p ) \\
    &=  (1-\delta)^w \cdot  \pr(f(X) \neq f(Y) \mid X^p = Y^p, X_{P(X^p)} = Y_{P(Y^p)} ) +\\
    & \hspace{0.8cm} +(1- (1-\delta)^w) \cdot  \pr(f(X) \neq f(Y) \mid X^p = Y^p, X_{P(X^p)} \neq  Y_{P(Y^p)} )\\
    & = (1- (1-\delta)^w) \cdot  \pr(f(X) \neq f(Y) \mid X^p = Y^p, X_{P(X^p)} \neq  Y_{P(Y^p)} )\\
    & = (1- (1-\delta)^w) \cdot \NS_\delta[g],
\end{align*}
where the last inequality holds because $g$ is unbalanced. By replacing $\NS_\delta[g] = \frac 12 - \frac 12 \Stab_\delta[g]$ and rearranging terms one can obtain~\eqref{eq:stabPVR}.

\subsection{Noise stability and initial alignment~\cite{AbbeINAL}}
\cite{AbbeINAL} introduced the notion of Initial Alignment (INAL) between a target function $f: \cX \to \cY$ and a neural network $\NN:\cX \to \cY$ with random initialization $\Theta_0$ and neuron set $V_{NN}$. The $\INAL$ is defined as 
\begin{align}
    \INAL(f, \NN) := \max_{v \in V_{\NN} } \E_{\Theta^0 } \E_{X} \left[ f(X) \cdot \NN_{\Theta^0}^{(v)}(X) \right]^2,
\end{align}
where $\NN_{\Theta^0}^{(v)} $ denotes the output of neuron $v$ of the network at initialization. In~\cite{AbbeINAL}, it is shown that GD cannot learn functions that have negligible initial alignment with a fully connected architectures with i.i.d. Gaussian initialization (with rescaled variance) and $\ReLU$ activation.
Here, we show how the $\INAL$ can be related to the noise sensitivity of the target function. We remark that both noise sensitivity and $\INAL$ are related by the cross-predictability (CP). Let us first give two definitions. Recall that for $f:\{\pm 1\}^n \to \{ \pm 1 \}$, $\NS_\delta[f] = \frac 12 -\frac 12 \Stab_\delta[f]$.
\begin{defin}[High-Degree.] 
We say that a family of functions $f_n:\{ \pm 1\}^n \to \mathbb{R}$  is ``high-degree'' if for any fixed $k$, $W^{\leq k} (f_n) $ is negligible. 
\end{defin}
\begin{defin}[Noise sensitive function]
We say that a family of functions $f_n:\{\pm 1 \}^n \to \{ \pm 1\}$ is noise sensitive if for any $\delta \in (0,1/2]$, $\NS_\delta [f_n] =  1/2 - o_n(1). $
\end{defin}

\begin{defin}[Strongly noise sensitive function]
We say that a family on functions $f_n:\{\pm 1 \}^n \to \{ \pm 1\}$ is strongly noise sensitive if for any $\delta \in (0,1/2]$, $\NS_\delta [f_n] =  1/2 - n^{-\omega(1)}. $
\end{defin}
Then we can prove the following.
\begin{prop} \label{prop:NS_INAL}
Let $\NN_{n}:\bR^n \to \bR$ be a fully connected neural network with Gaussian i.i.d. initialization and expressive activation (as in Theorem 2.7 in~\cite{AbbeINAL}). If $\INAL(\NN_n,f_n) = n^{-\omega(1)}$, then $f_n$ is noise sensitive.
\end{prop}
\begin{proof}
We need to show that for any $\delta \in [0,1/2]$, $\sum_{k=0}^n (1-2\delta)^k W^k(f_n) = o_n(1)$, or analogously that for any $\epsilon>0$ and for $n $ large enough $ \sum_{k=0}^n (1-2\delta)^k W^k(f_n) < \epsilon$.  Fix $\delta$ and let $\epsilon >0$. Let $k_0$ be such that $(1-2\delta)^{k_0} < \epsilon/2$. Then,
\begin{align}
    \sum_{k=0}^n (1-2\delta)^k W^k(f_n)  &= \sum_{k=0}^{k_0} (1-2\delta)^k W^k(f_n)  + \sum_{k=k_0+1}^{n} (1-2\delta)^k W^k(f_n) \\
    & \leq W^{\leq k_0}(f_n) + (1-2\delta)^{k_0 + 1} \sum_{k=k_0 + 1}^n W^k(f_n).
\end{align}
By Proposition 4.3 and Corollary 4.4 in~\cite{AbbeINAL}, if $\INAL(f_n,\sigma ) =n^{-\omega(1)}$ then $f_n$ is high degree. Thus, $W^{\leq k_0}(f_n) =n^{-\omega(1)} $, and clearly for $n$ large enough $W^{\leq k_0}(f_n) <\epsilon/2 $. On the other hand, $\sum_{k=k_0+1}^n W^k(f_n) <1$, since $f $ is Boolean-valued. Thus, $\sum_{k=0}^n (1-2\delta)^k W^k(f_n)<\epsilon $, and the Proposition is proven.
\end{proof}

\begin{prop} \label{prop:NS_high_degree}
If $f_n $ is strongly noise sensitive, then $f_n$ is high degree.
\end{prop}
\begin{proof}
We need to show that if $\sum_{k=0}^n (1-2\delta)^k W^k(f_n) = n^{-\omega(1)} $ then for any constant $k$, $W^{\leq k}(f_n) = n^{-\omega(1)} $. Take $k_0 \in \bN$, then
\begin{align}
    n^{-\omega(1)} &= \sum_{k=0}^n (1-2\delta)^k W^k(f_n) \geq \sum_{k=0}^{k_0} (1-2\delta)^k W^k(f_n)  \geq (1-2\delta)^{k_0} W^{\leq k_0}(f_n).
\end{align}
Clearly this implies that $W^{\leq k_0}(f_n) = n^{-\omega(1)}$, and the proof is concluded.
\end{proof}


\section{Computation of the Boolean influence for PVR functions} \label{sec:BoolInfForPVR}
In this section, we compute the Boolean influence for PVR functions. Here, we consider PVR functions with sliding windows and cyclic indices (i.e., $x_{n+1} = x_{p+1}$). The Boolean influence for PVR tasks with truncated windows or non-overlapping windows can be calculated in a similar manner. Also note that we never freeze pointer bits in this paper as done in \cite{Zhang2021PointerVR}; therefore, we skip the calculation of the Boolean influence for pointer bits. Consider a bit at $k$-th position ($k > p$). Note that this bit appears in $w$ windows. We denote by $U^n$ the uniform distribution over the $n$-dimensional hypercube. Using Lemma~\ref{lem:Boolean_inf}, we have: 
\begin{align}  
    \mathrm{Inf}_{k}(f) &= \mathbb{E}_{x \sim U^n} \frac{1}{2} \left(f\left(x\right)-f_{-k}\left(x\right) \right)^2 \\
    &= \mathbb{E}_{x \sim U^n} \frac{1}{2} \Biggl(\sum_{i=0}^{w - 1}\mathds{1}(P(x^p) = k-i)\Bigl(g\left(x_{k-i}, \ldots, x_{k}, \ldots, x_{k-i+w - 1}\right) \\ &\phantom{= \mathbb{E}_{x \sim U^n} \frac{1}{2}} - g\left(x_{k-i}, \ldots, 1, \ldots, x_{k-i+w - 1}\right)\Bigr) \Biggr)^2 \nonumber \\
    &= \mathbb{E}_{x \sim U^n} \frac{1}{2} \Biggl(\sum_{i=0}^{w - 1}\mathds{1}(P(x^p) = k-i)\Bigl(g\left(x_{k-i}, \ldots, x_{k}, \ldots, x_{k-i+w - 1}\right) \\ &\phantom{= \mathbb{E}_{x \sim U^n} \frac{1}{2}} - g\left(x_{k-i}, \ldots, 1, \ldots, x_{k-i+w - 1}\right)\Bigr)^2 \Biggr) \nonumber \\
    &= \frac{1}{2^p}\sum_{i=1}^{w} \mathrm{Inf}_{i}(g) \label{eq:total_inf}.
\end{align}
Note that the expression $\sum_{i=1}^{w} \mathrm{Inf}_{i}(g)$ in Equation (\ref{eq:total_inf}) is known as the \textit{total influence} of the aggregation function $g$~\cite{o'donnell_2014}.  Below follows the value of the Boolean influence of the PVR task $f$, depending on different aggregation functions:
\begin{itemize}
    \item \textbf{Parity.} If we choose $g$ to be the parity function, i.e., $g(x_1, \ldots, x_w) = x_1x_2\cdots x_w$ then $\mathrm{Inf}_i(g) = \mathbb{P}(g(x) \neq g(x + e_i)) = 1$. Therefore, $\mathrm{Inf_k}(f) = \frac{w}{2^p}$. 
    \item \textbf{Median/Majority vote.} We define the majority vote function as $g(x_1, \ldots, x_w) = \mathrm{sign}(x_1 + \cdots + x_w)$ where the sign function outputs $+1$, $-1$, and $0$. First assume $w$ is odd. In this case, flipping the $i$-th bit matters only in the case where exactly $\frac{w-1}{2}$ other bits have the same sign as the $i$-th bit. Therefore, $\mathrm{Inf}_i(g) =\mathbb{P}(g(x) \neq g(x + e_i))= 2^{-(w-1)}\binom{w-1}{\frac{w-1}{2}}$. Similarly, if $w$ is even, flipping the $i$-th bit only matters if there are exactly $\frac{w}{2}$ or $\frac{w}{2} - 1$ other bits with the same sign. Using Lemma~\ref{lem:Boolean_inf}, $\mathrm{Inf}_i(g) = \mathbb{E}_{x \sim U^w} \frac{1}{2} \left(g\left(x\right)-g_{-i}\left(x\right) \right)^2 =  2^{-(w+1)}\left(\binom{w-1}{\frac{w}{2}}+\binom{w-1}{\frac{w}{2} - 1}\right) = 2^{-w}\binom{w-1}{\frac{w}{2}}$. Therefore, for odd $w$,  
    $\mathrm{Inf_k}(f) = \frac{w}{2^{(p+w-1)}}\binom{w-1}{\frac{w-1}{2}}$ and for even $w$,
    $\mathrm{Inf_k}(f) = \frac{w}{2^{(w+p)}}\binom{w-1}{\frac{w}{2}}$.
    
    \item \textbf{Min/Max.} Here we consider the $\min$ function, $g(x_1, \ldots, x_w) = \min(x_1, \ldots, x_w)$. By symmetry, the Boolean influence values are the same for the $\max$ function. In this case, flipping the $i$-th bit only matters if all bits other than $x_i$ are equal to $+1$. Thus, the Boolean influence is given by $\mathrm{Inf}_i(g) = 2^{-(w-1)}$ and hence, $\mathrm{Inf_k}(f) = \frac{w}{2^{(p + w - 1)}}$.
\end{itemize}
One can see how different parameters of the Boolean PVR functions, such as $p$, $w$, and $g$ affect the Boolean influence. Assuming fixed window size, $w$, each bit is less likely to appear in a window if the number of pointer bits, $p$, is increased. Hence for fixed $w$ and $g$, an increase in $p$ results in smaller influence for all the bits. On the other hand a change of $w$ has a two-fold effect. First, since each bit appears in $w$ windows, the increase of $w$ makes each bit more likely to appear in a window. On the other hand, for some functions such as majority-vote and min/max, the increase of $w$ reduces the Boolean influence of the aggregation function for all bits. Thus, the increase of $w$ can result in either an increase of the Boolean influence (for example, if parity is used) or a decrease of the Boolean influence (for instance, if min/max aggregation is used). 
We refer to Appendix~\ref{appendix:experiments} for experiments on PVR tasks with varying window size.

\section{Experiment Details and Additional Experiments}\label{appendix:experiments}
In this section, we describe the experiments in more detail. Furthermore, we demonstrate more experiments on the comparison of the out-of-distribution generalization error and the Boolean influence. 
\subsection{Architectures and Procedure}
We first explain the general experimentation setup for PVR tasks and other functions. Afterward, we describe the procedure used for linear neural networks and results presented in Section~\ref{sec:linearmodels}.

\paragraph{Architectures.}Three architectures have been used for the main experiments of this paper: MLP, the Transformer \cite{vaswani2017attention-transformer}, and MLP-Mixer \cite{tolstikhin2021mlpmixer}. Below, we describe each of these architectures:
\begin{itemize}
    \item \textbf{MLP.} The MLP model consists of 4 fully connected hidden layers of sizes 512, 1024, 512, and 64. We used $\mathrm{ReLU}$ as the activation function for all layers except the last layer.
    \item \textbf{Transformer.} 
    We follow the standard decoder-only Transformer architectures~\citep{raffel2019exploring} that are commonly used for language modeling, and are also the backbone of Vision Transformers (ViTs)~\citep{dosovitskiy2020image}. 
    Specifically, an embedding layer is used to embed the binary $+1$ and $-1$ values into 256 dimensional vectors, and a shared embedding layer is used for all the binary tokens in the input sequence. Then, the embedded input is passed through 12 transformer layers \cite{vaswani2017attention-transformer}. In each transformer layer, the hidden dimension of MLP block is also 256. Moreover, 6 heads are used for each self-attention block. At the end, a linear layer is used to compute the output of the model.
    \item \textbf{MLP-Mixer.} Similar to the Transformer based model, first we embed $+1$ and $-1$ tokens into a 256 dimensional vector using a shared embedding layer for all the binary input tokens. Then, the embedded input is passed through a standard 12-layer MLP-Mixer model~\cite{tolstikhin2021mlpmixer}. Finally, a linear layer is used to compute the output. The MLP-Mixer architectures are similar to the decoder-only Transformers, except that ``mixer layers'' based on MLPs are used instead of the attention mechanism. Please see \cite{tolstikhin2021mlpmixer} for details.
\end{itemize}

\paragraph{Procedure.}To perform each of the experiments, we first fix a dimension to be frozen during the training. Afterward, we train the model on the frozen training set to make the model learn the frozen function. Finally, we evaluate the trained model uniformly on the Boolean hypercube ($\{\pm 1\}^n$) to compute the out-of-distribution generalization error.   

Now, we explain the hyperparameters used for the experiments. Note that the experiments are aimed to exhibit an implicit bias towards low degree monomials and consequently to show that the generalization error is close to the Boolean influence. Therefore, the experiments are not focused on the learning of the frozen function itself and the in-distribution generalization error. In other words, we are interested in the setting that the frozen function is learned during the training, and then we want to examine the out-of-distribution generalization. Due to this reason, we have always used a relatively large number of training samples. Moreover, we have not used a learning rate scheduler and we have not done extensive hyperparameter tuning. Generally, we have considered two optimizers for the training of the models: mini-batch SGD and Adam \cite{kingma2014adam}. We describe the setting used for each of them below:
\begin{itemize}
    \item \textbf{SGD.} When using SGD, we tried using $0$ and $0.9$ momentum. We observed that the use of momentum remarkably accelerates the training process, and hence we continued with $0.9$ momentum. Nonetheless, we also performed experiments using SGD without momentum, and we did not notice any difference in their preference towards low degree monomials and therefore the out-of-distribution generalization error. For learning rate of SGD, we tried values in $\{10^{-3}, 5\times10^{-4}, 2.5\times 10^{-4}, 10^{-4}\}$ and selected the learning rate dependent on the model and task (more on this below). Additionally, we always used the mini-batch size of 64 with SGD. 
    \item \textbf{Adam.} In our experiments with Adam, we used default values of the optimizer and only changed the learning rate.
For learning rate, we tried values in $\{10^{-4}, 5\times10^{-5}, 10^{-5}\}$ and finally selected $5\times 10^{-4}$. While employing Adam, we used mini-batch size of 64 for PVR tasks with 3 pointer bits (11 bits in total) and mini-batch size of 1024 for PVR tasks with 4 pointer bits (20 bits in total). 
\end{itemize}
 We selected the learning rate (and in general, hyperparameters) based on the speed of the convergence and its stability. Note that we set the number of epochs for each task to a value to ensure the training loss and in-distribution generalization error are small enough.\footnote{This is problem dependent; nonetheless, we generally refer to errors of order of magnitude $10^{-2}$ or less.}

Finally, we note that all of our experiments are implemented using PyTorch framework \cite{torch}, and the training has been done on NVIDIA A100 GPUs. The experiments presented in this paper took approximately 250 GPU hours.  
Note that we have repeated PVR experiments with 3 pointer bits 40 times and the rest of the experiments 20 times, and have reported the averaged results and $95\%$ confidence interval.
Please refer to the code for more details on the experiments.

\paragraph{Linear neural networks.} For the experiments on linear models, we considered fully connected linear neural networks with fixed hidden layer size of 256. As presented in Figure \ref{fig:linear}, we varied the initialization and depth of these networks. For optimizing linear neural networks, we used mini-batch SGD with 64 and $10^{-5}$ as the batch size and learning rate respectively. Note that we trained linear models on CPU and stopped the training when the loss became less than $10^{-8}$.

\begin{figure}[h]
%
     \centering
     \begin{subfigure}[b]{0.75\textwidth}
         \centering
         \includegraphics[width=\textwidth]{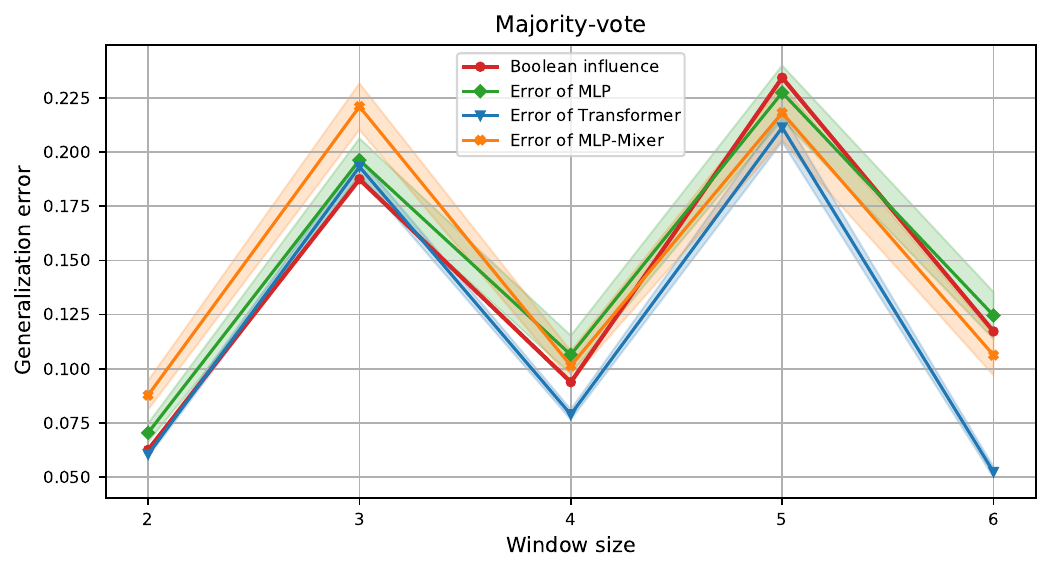}
     \end{subfigure}
     \hfill
     \begin{subfigure}[b]{0.5\textwidth}
         \centering
         \includegraphics[width=\textwidth]{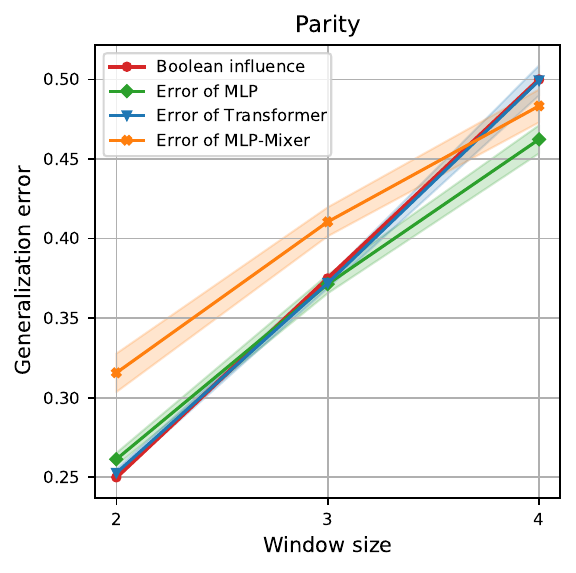}
     \end{subfigure}
     \hfill
        \caption{PVR tasks with 3 pointer bits and varying window sizes where the aggregation function is majority-vote (top) and parity (bottom). X-axis represents the window size of the PVR task, and y-axis shows the value of the generalization error and the Boolean influence. 
        }
        \label{fig:majorityparity}
\end{figure}
\subsection{Additional Results}
\paragraph{More PVR tasks.} First, we compare the generalization error and the Boolean influence for more PVR functions. In the additional experiments, we consider the cyclic version of the PVR (i.e., $x_{n+1} = x_{p+1}$), and due to the symmetry, we only freeze one dimension of the input. Also, we use Adam to optimize the models (instead of SGD) due to faster convergence. As a first example, we consider PVR tasks with 3 pointer bits (11 bits in total) and varying window sizes. We use majority-vote and parity as the aggregation functions (see Appendix \ref{sec:BoolInfForPVR} for computation of the Boolean influence for such functions). 
In Figure \ref{fig:majorityparity}, the window size of the aforementioned PVR tasks is varied in the x-axis and the averaged generalization error over 40 experiments is shown. Figure \ref{fig:majorityparity} (top) corresponds to the case where majority-vote is used as the aggregation function whereas Figure \ref{fig:majorityparity} (bottom) shows the results when parity is the aggregation function. It can be seen that in this setting, the generalization error of all models follow the Boolean influence closely. Note that learning parity function becomes increasingly difficult as the window size is increased. Even for $w=4$, the MLP and MLP-Mixer models could not learn the frozen function completely and their in-distribution generalization loss was between $0.05$ and $0.10$.

Furthermore, we experimented on PVR tasks of larger scales. To this end, we consider PVR tasks with 4 pointer bits (20 bits in total) and different window sizes and aggregation functions. For these experiments, we also used Adam optimizer with batch-size of 1024. We repeated each experiment 20 times. The generalization error and Boolean influence for these functions are given in Table \ref{table:pvrs}. It can be observed that for these experiments, the generalization errors of MLP and Transformer are well approximated by the Boolean influence; while MLP-Mixer has higher generalization error. 
The results of Figure \ref{fig:majorityparity} and Table \ref{table:pvrs} indicate that the implicit bias towards low-degree monomials 
also exists when Adam is used as the optimizer and therefore is not limited to SGD.

\begin{table}[H]
  \caption{Generalization error for PVR tasks with 4 pointer bits}
  \vspace{1em}
  \label{table:pvrs}
  \centering
  \begin{tabular}{cccccc}
    \toprule
    \multicolumn{2}{c}{PVR task}& & \multicolumn{3}{c}{Generalization error}                   \\
    \cmidrule(r){1-2} \cmidrule(r){4-6}
    Aggregation & Window size & \begin{tabular}{@{}c@{}}Boolean \\ influence\end{tabular} & MLP     & Transformer & MLP-Mixer \\
    \midrule
    Min & $2$ & $0.0625$ & $0.062 \pm 0.004$ & $0.068 \pm 0.006$ & $0.118 \pm 0.016$ \\
    Parity & $3$ & $0.1875$ & $0.206 \pm 0.004$ & $0.198 \pm 0.015$ & $0.329 \pm 0.017$ \\
    Majority & $3$ & $0.09375$ &  $0.099\pm 0.004$ & $0.095\pm 0.001$ & $0.194 \pm 0.022$ \\
    Majority & $4$ & $0.046875$ &  $0.051 \pm 0.004$ & $0.049 \pm 0.002$ & $0.094 \pm 0.019$ \\
    
    
    \bottomrule
  \end{tabular}
\end{table}
\paragraph{Non-PVR examples.} We also experimented on non-PVR functions. As the first example, we consider the target function $f_1(x_1, \ldots, x_{11}) = x_1x_2 + 2x_2x_3 + 3x_3x_4 +4x_4x_5+ \cdots + 10x_{10}x_{11}$ which is a sum of second degree monomials. For each of the architectures, we freeze a coordinate (ranging from $1$ to $11$), train the model on the frozen samples using mini-batch SGD and evaluate the generalization loss. The relation between the Boolean influence and the averaged generalization error over 20 runs for $f_1$ is demonstrated in Figure \ref{fig:f1}. It can be seen that the generalization errors of the MLP and the Transformer model are again well approximated by the Boolean influence. However, the generalization error of the MLP-Mixer follows the trend of Boolean influence with an offset. This implies that the MLP and Transformer have a stronger preference for low-degree monomials than the MLP-Mixer in this case. 
\begin{figure}[t]
    \centering
    \includegraphics[width=0.6\textwidth]{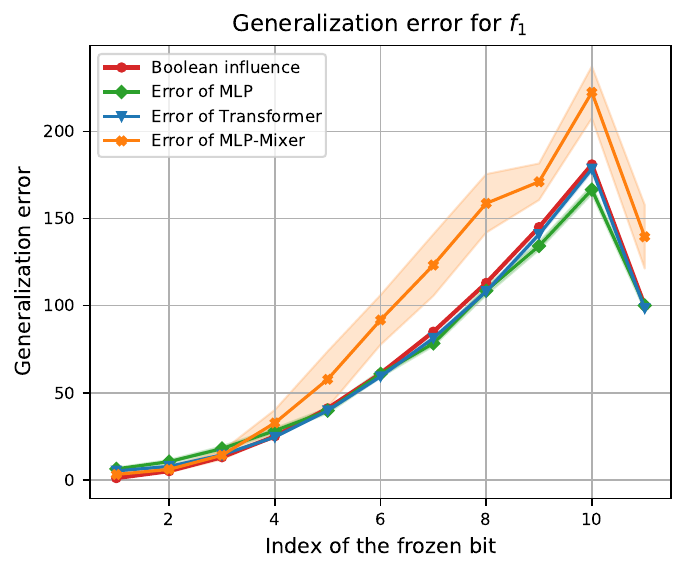}
    \caption{Comparison between the Boolean Influence and generalization error for $f_1(x_1, \ldots, x_{11}) = x_1x_2 + 2x_2x_3 + 3x_3x_4 + \cdots + 10x_{10}x_{11}$. Frozen coordinates are represented by the x-axis; while the y-axis represents the value of generalization error and the Boolean influence.}
    \label{fig:f1}
\end{figure}
\begin{figure}[t]
    \centering
    \includegraphics[width=0.6\textwidth]{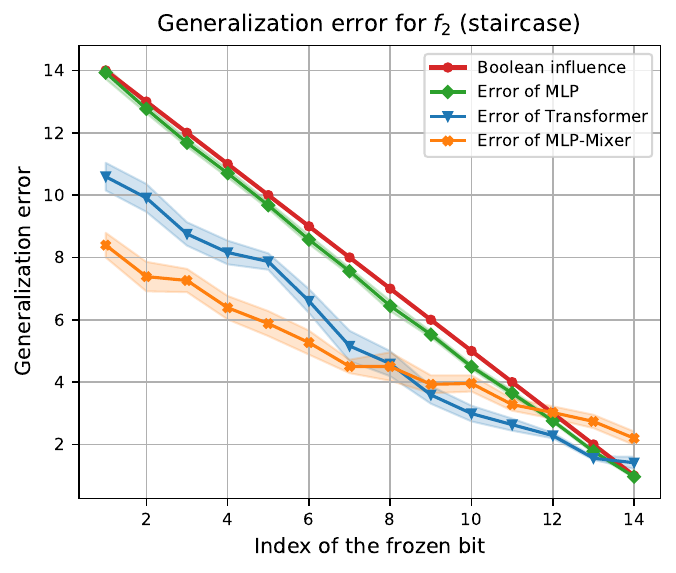}
    \caption{Comparison between the generalization loss in the canonical distribution shift setting and the Boolean Influence for $f_2(x_1, x_2, \ldots, x_{14}) = x_1 + x_1x_2+x_1x_2x_3 + \cdots + x_1x_2x_3\cdots x_{14}$.}
    \label{fig:f2}
\end{figure}
As the last example, we consider the vanilla staircase function for 14 bits,
i.e., $f_2(x_1, x_2, \ldots, x_{14}) = x_1 + x_1x_2 + x_1x_2x_3+ \cdots + x_1x_2x_3\cdots x_{14}$ (see \cite{abbe2021staircase, mergedstaircase} for theoretical results on such staircase functions). We train models for this function using mini-batch SGD. In Figure~\ref{fig:f2}, we report the generalization errors of the MLP, Transformer, and MLP-mixer models for each frozen coordinate of $f_2$, as well as the values of the Boolean influence of the corresponding index. Note that the generalization errors have been averaged over 20 runs. It can be observed that the generalization loss of MLP is very close to the Boolean influence in this case as well. However, the generalization errors of the Transformer and MLP-Mixer follow the Boolean influence with an offset. 
\begin{figure}[t]
     \centering
     \begin{subfigure}{0.49\textwidth}
         \centering
         \includegraphics[width=\textwidth]{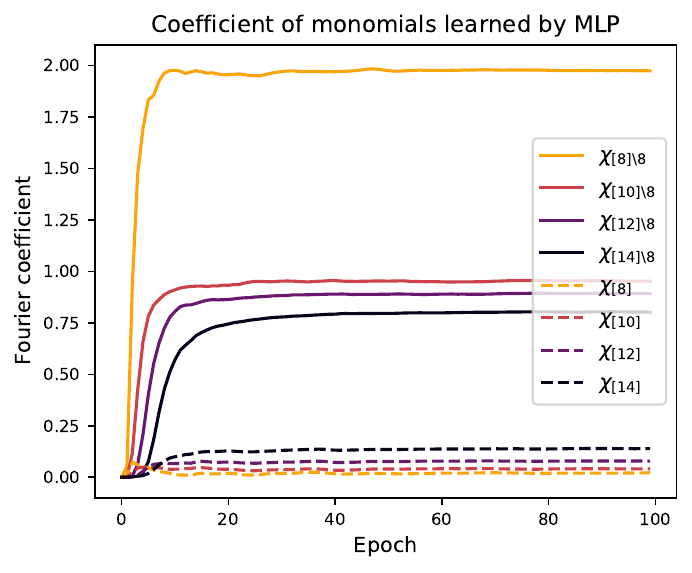}
         \caption{MLP}
     \end{subfigure}
     \hfill
     \begin{subfigure}{0.49\textwidth}
         \centering
         \includegraphics[width=\textwidth]{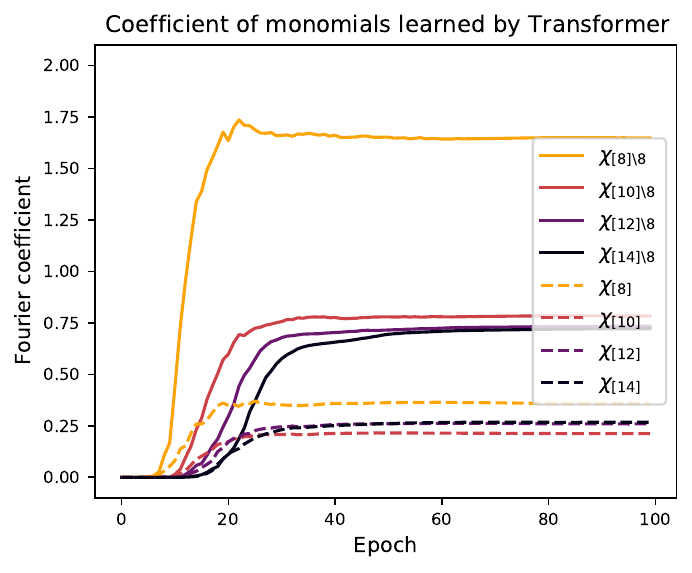}
         \caption{Transformer}
     \end{subfigure}
     \hfill
     \begin{subfigure}{0.49\textwidth}
         \centering
         \includegraphics[width=\textwidth]{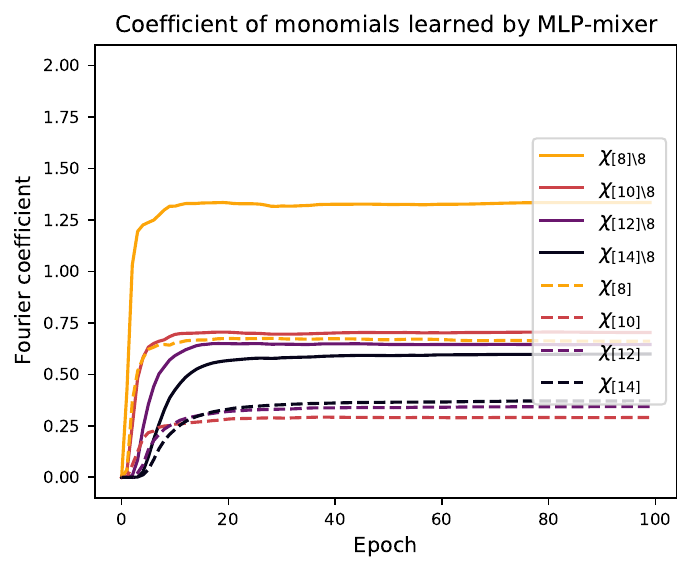}
         \caption{MLP-Mixer}
     \end{subfigure}
        \caption{The coefficients of different monomials learned by the MLP, Transformer, and MLP-Mixer when learning the staircase function, with $x_{8}=1$ frozen during the training. Note that $\chi_T(x) = \prod_{i \in T} x_i$, e.g., $\chi_{[8]} = x_1x_2\cdots x_8$. Generally, monomials of lower degrees are learned faster by the models. Consequently, the models prefer to learn the monomials which exclude the frozen index.}
        \label{fig:speed-staircase}
\end{figure}
It is worth noting that the previous two functions are quite different than the PVR function: the PVR function has strong `symmetries' given by the fact that each window is treated similarly with the aggregation function, and thus one may expect that certain architectures would exploit such symmetries. Thus, the PVR function is still a staircase function  \cite{mergedstaircase}, of leap 2 in the example of Section \ref{sec:lowdegreebias}, but it is a staircase function with the related symmetries. Instead, the two functions considered here are staircase functions that do not have any such symmetries.

Figure \ref{fig:f2} shows that in some cases the Boolean influence may not always give a tight characterization of the generalization error. However, it appears that in such cases the offset still maintains the general trend of the influence. As an attempt to better understand this offset, recall the relation between the Boolean influence and the generalization error in terms of the implicit low-degree bias: the stronger the preference for low-degree monomials is, the closer the generalization error is to the Boolean influence. We thus plot the coefficient of different monomials for $f_2$ and these three models while $x_{8} = 1$ is frozen during the training in Figure \ref{fig:speed-staircase}. One can observe that for this staircase function, the bias towards low-degree monomials is stronger for MLP and it is weaker for Transformer and MLP-mixer. This well explains the relation between the Boolean influence and the generalization error of different models depicted in Figure~\ref{fig:f2}, where the generalization error of MLP is significantly closer to the Boolean influence, compared to the generalization error of Transformer and MLP-Mixer.

\section{Intuition on the linear neural networks}
\label{app:linear-nns}
At last, we provide heuristic justifications for the effect of depth and initialization on the generalization error and its closeness to the Boolean influence in the case of linear neural networks. Let  $f_{NN}(x;\Theta) = w_L^T(W_{L-1}^T(\cdots(W_1^Tx+b_1)\cdots)+b_{L-1}) + b_L$ be a linear neural network with depth $L$, after training in the canonical holdout setting where the $k$-th bit is frozen to $1$. Assume the target function to be linear. After training, the neural network learns the frozen function $f_{-k}(x) = f(x_{-k})$. Note that the bias of the frozen function is $\hat f(\{\emptyset \}) + \hat f(\{k\})$ (where with $\hat f$ we denote the Fourier coefficients of the target function $f$), that is expressed by the neural network by the following: 
\begin{align} \label{eq:biasB}
    B:= (w_L^TW_{L-1}^T\cdots W_2^Tb_1 + w_L^TW_{L-1}^T\cdots W_3^Tb_2 + \cdots +w^T_L &b_{L-1} + b_L)+  w_L^TW_{L-1}^T\cdots W_2^Tw_{1,k}^T
\end{align} 
where by $w_{1,k}$ we indicate the weights in the first layer of the frozen dimension $k$. Assuming the neural network has learned the function, we have 
\begin{align}
    &\hat f_{NN}(\{i\}) =  \hat f(\{i\}) \qquad \text{for all $i \neq k$},\\
    &  \hat f_{NN}(\{\emptyset\}) + \hat f_{NN}(\{k\})  = B =   \hat f(\{\emptyset\})+ \hat f(\{k\}),
\end{align}
where we denoted by $\hat f_{NN}$ the Fourier coefficients of $f_{NN}$. Therefore, applying Parseval identity we find that the generalization error equals
\begin{align}
\mathrm{gen}(f, f_{NN}) & = \frac 12 \E_X (f(X) - \hat f_{NN}(X))^2 \\
& = \frac 12 \left(\hat f(\{\emptyset \} )  - \hat f_{NN}(\{\emptyset \} )\right)^2 + \frac 12\left(\hat f(\{ k \} )  - \hat  f_{NN}(\{ k\} )\right)^2  \\
&= \frac{1}{2}\left(\hat f(\{\emptyset\}) - (w_L^T\cdots W_2^Tb_1  + \cdots + b_L)\right)^2 + \frac 12 \left(\hat f(\{k\}) - w_L^T \cdots W_2^Tw_{1,k}^T\right)^2 \\&= (\hat f(\{k\}) - w_L^T \cdots W_2^Tw_{1,k}^T)^2.	
\end{align}
Therefore, the amount of bias captured by $w_L^T W_{L-1}^T \cdots W_2^Tw_{1,k}^T$ determines the generalization error. Particularly, if $w_L^T \cdots W_2^T w_{1,k}^T$ goes to zero, the generalization error will become equal to the Boolean influence. Note that $x_k=1$ during the training, therefore $w_{1,k}$ has the same training dynamics as the bias of the first layer $b_1$. 

\paragraph{Effect of depth.} From~\eqref{eq:biasB}, we note that there are $L+1$ terms that contribute to $B$, and one of them is indeed $w_L^T \cdots W_2^T w_{1,k}^T$. Therefore as the depth $L$ increases, if those terms are appropriately aligned, one can expect that the contribution of each term, including $w_L^T \cdots W_2^T w_{1,k}^T$, decreases; thus, the generalization error becomes closer to the Boolean influence. 

\paragraph{Effect of initialization.} The gradients of the parameters for a sample $x$ are given by
\begin{align*}
    \nabla_{b_L}L(\Theta, x, f) &= (f_{NN}(x;\Theta)(x) - f_{-k}(x)), \\
    \nabla_{b_{L-1}}L(\Theta, x, f) &= (f_{NN}(x;\Theta)(x) - f_{-k}(x))w_{L}, \\
    &~~\vdots \\
    \nabla_{b_{1}}L(\Theta, x, f) &= (f_{NN}(x;\Theta)(x) - f_{-k}(x))W_2W_3\cdots w_{L}.
\end{align*}
Now, consider the first update of the parameters. As we decrease the scale of initialization, the ratio of  $\frac{\nabla_{b_L}}{\nabla_{b_{L-1}}}$, $\cdots$, $\frac{\nabla_{b_2}}{\nabla_{b_{1}}}$ increases which implies that $b_1$ would have the smallest update and $b_L$ will have the largest update. Since the dynamics of $w_{1, k}$ and $b_1$ are the same, the frozen dimension would contribute the least to the bias after the first iteration. Our experiments on decreasing the scale of initialization suggest that this argument is not limited to the first iteration. In other words, using small enough initialization the bias will be mostly captured by the bias terms in other layers, which results in generalization error being close to the Boolean influence.


\end{document}